\documentclass[twoside,11pt,fleqn]{article}

\usepackage{fullpage}
\usepackage{theapa}

\usepackage{amssymb}
\usepackage{ulem} 
\renewcommand{\emph}[1]{\textit{#1}}
\usepackage{algorithm}
\usepackage{algorithmic}
\usepackage{url}
\usepackage{tikz}
\usetikzlibrary{arrows,automata}

\renewcommand{\algorithmiccomment}[1]{\hfill// #1}

\usepackage{longtable}

\newcommand{\MOD}[0]{\textit{Mod}}
\newcommand{\beh}[0]{\sigma}
\newcommand{\OBS}[0]{\textit{o}}

\newcommand{\hypo}[0]{\textit{hypo}}
\newcommand{\hypos}[0]{\textit{hypos}}
\newcommand{\hypospace}[0]{\mathbb{H}}
\newcommand{\props}[0]{\textit{props}}
\newcommand{\propspace}[0]{\mathbb{P}}
\newcommand{\propindex}[1]{\footnotesize #1}
\newcommand{\propdesc}{p_{\textrm{\propindex{desc}}}}
\newcommand{\propanc}{p_{\textrm{\propindex anc}}}
\newcommand{\descendants}[0]{\mathrm{descendants}}

\newcommand{\children}[0]{\mathrm{children}}
\newcommand{\ancestors}[0]{\mathrm{ancestors}}

\newcommand{\projhypof}[0]{\alpha}
\newcommand{\projhypo}[1]{\projhypof(#1)}

\newcommand{\setspace}[0]{\mathbb{S}}

\newcommand{\candidate}[1]{\mathrm{candidate}(#1)}
\newcommand{\minimal}[1]{\mathrm{minimal}(#1)}
\newcommand{\coverage}[1]{\mathrm{covers}(#1)}

\newtheorem{defi}{Definition}

\newtheorem{theo}{Theorem}
\newtheorem{lemm}{Lemma}
\newtheorem{exam}{Example}

\newcounter{myqnum}
\newenvironment{question}%
{\begin{trivlist}
 \refstepcounter{myqnum}
 \item[\hskip\labelsep {\bfseries Question \arabic{myqnum}.}]
}%
{\end{trivlist}}

\newenvironment{proof}[1][]{\noindent\textbf{\ifthenelse{\equal{#1}{}}{Proof: }{Proof #1: }}}{\hfill{\scriptsize $\Box$}}

\newcommand{\SO}[0]{S_{\mathrm{O}}}
\newcommand{\SR}[0]{S_{\mathrm{R}}}

\newcommand{\set}[1]{\{#1\}}

\newcommand{\JT}[0]{\textsc{jt}}
\newcommand{\BDD}[0]{\textsc{bdd}}
\newcommand{\PFS}[0]{\textsc{pfs}}
\newcommand{\PLS}[0]{\textsc{pls}}

\usepackage{soul}

\begin{document}

\title{A More General Theory of Diagnosis from First Principles}

\author{Alban Grastien, Patrik Haslum, Sylvie Thi\'ebaux\\
  The Australian National University, Canberra, Australia\\
  \texttt{first.last@anu.edu.au}
}
\date{September 2023}

\maketitle

\begin{abstract}
  Model-based diagnosis has been an active research topic in different communities including artificial intelligence, formal methods, and control. This has led to
a set of disparate approaches addressing different classes of systems and seeking different forms of diagnoses. For instance Reiter's ``Theory of Diagnosis from First Principles'' primarily targets static systems, considers that diagnoses are minimal sets of faults consistent with the system's model and the observation, and efficiently explores the powerset of faults by means of simple consistency tests.
In contrast, diagnosis approaches to discrete event dynamic systems, pioneered by Sampath, Zanella, and others, traditionally reconstruct all system traces consistent with the observation, either explicitly or through a precompiled structure.
In this paper, we resolve such disparities by generalising Reiter's theory to be agnostic to the types of systems and diagnoses considered. This more general theory of diagnosis from first principles defines the minimal diagnosis as the set of preferred diagnosis candidates in a search space of hypotheses. Computing the minimal diagnosis is achieved by exploring the space of diagnosis hypotheses, testing
sets of hypotheses for consistency with the system's model and the observation, and generating conflicts that rule out successors and other portions of the search space. Under relatively mild assumptions, our algorithms correctly compute the set of preferred diagnosis candidates. The main difficulty here  is that the search space is no longer a powerset  as in Reiter's theory,  and that, as consequence, many of the implicit properties  (such as finiteness of the search space) no longer hold.  The notion of conflict also needs to be generalised and we present such a more general notion.
We present two implementations of these algorithms, using test solvers based on satisfiability and heuristic search, respectively, which we evaluate on instances from two real world discrete event problems. Despite the greater generality of our theory, these implementations surpass the special purpose algorithms designed for discrete event systems, and enable solving instances that were out of reach of existing diagnosis approaches.

\end{abstract}

\section{Introduction}
Discrete event systems \cite{cassandras-lafortune::99} (DESs)
are models of dynamic systems
that represent states and events in a discrete manner.
DESs are a natural model of many kinds of event-based systems,
such as, for example, protocols \cite{holzmann::1991} or business
processes \cite{van-der-aalst::survey::2013}, and also often form
a natural abstraction of hybrid discrete--continuous dynamical
systems.
The diagnosis problem, in the context of dynamical systems, is to
infer from a system model and partial observation of events emitted
by the system some diagnostically relevant properties of its current
state or behaviour -- for example, whether any abnormal events have
occurred, and if so, which ones, how many times and in what order?

Since the seminal work of Sampath et al.\ \cite{sampath-etal::tac::95},
DESs diagnosis methods have examined all sequences of events that represent possible system behaviours under the system model and the observation,
and have extracted the diagnostic information from those sequences.

This contrasts with the approach developed by
the Artificial Intelligence 
community for static systems:
known as ``diagnosis from first principles'' 
(i.e., model-based diagnosis, as opposed to expert-based diagnosis)
the approach pioneered by de Kleer, Reiter and Williams \cite{reiter::aij::87,dekleer-williams::aij::87} uses a theorem prover
to test the consistency of diagnostic hypotheses 
with the model and 
the observation.
By working directly at the level of hypotheses relevant to the diagnosis,
this approach
avoids enumerating all explanations of the observation (which are, in general,
exponentially many).

When trying to understand why such a ``test-based'' diagnosis approach for DESs did not eventuate, two main reasons come to mind.
The first is the absence of an efficient ``theorem prover'' for checking
the consistency of a set of hypotheses and an observed DES,
which is a problem akin to 
planning or model checking.
However, there has been considerable work 
in these areas in the last decades 
so that available tools can now be used for diagnosis
(cf., \cite{grastien-etal::aaai::07,sohrabi-etal::kr::10,%
haslum-grastien::spark::11}).
The second reason is that 
the \texttt{diagnose} algorithm proposed by Reiter \cite{reiter::aij::87}
was designed
to diagnose
circuits, and therefore returns only a set of faults.
DESs, in contrast, can experience multiple occurrences of the same
fault event,
and the diagnoser may be required to determine the number of repetitions of
faults, or order in which they took place.
Reiter's algorithm cannot be applied in this setting
and extending it in this direction raises major issues.
Our 
main contribution in this paper is to 
resolve these issues and 
generalise the
test-based diagnosis framework to a larger class of diagnostic hypothesis
spaces, appropriate to DESs and other models of dynamical systems.

We present a general definition of model-based diagnosis, independent
of the form of the system model and the form of diagnosis required.
This definition encompasses the existing theory of diagnosis of
circuits as a special case, but also applies to dynamic system models,
such as DESs, and beyond. 
As a result, DES diagnosis problems can be solved 
using the same techniques as for circuit diagnosis.  

More precisely, we formulate the diagnosis problem as follows: 
given a set of \textit{hypotheses} 
(abstractions of the system behaviour that discriminate only 
according to aspects that are relevant to the diagnosis)
and a preference relation over the hypotheses, 
the diagnosis is defined as the set of minimal (most-preferred) \textit{diagnosis candidates}, where a candidate is a hypothesis that is consistent with the model and the observation.
\textit{Diagnosis} is therefore 
the problem of exploring the \textit{hypothesis space} 
to identify these minimal diagnosis candidates.
We present different \textit{exploration strategies} 
that require only an oracle capable of testing
whether a given set of hypotheses intersects the diagnosis.
This test solver plays a role similar to the theorem prover in
Reiter's algorithm.
Importantly, we show that the test solver does not have to be
given an explicit, enumerated set of hypotheses. Instead, the
set of hypotheses to test is implicitly represented as those
that satisfy a set of \textit{diagnostic properties}; the test
solver's task is then to find a candidate that satisfies these
properties.
The implicit representation of hypothesis sets allows the diagnosis
algorithm to test infinite sets of hypotheses that can be represented
by a finite set of properties.

The exploration strategies we propose fall into two classes:
The ``preferred-first'' strategies start by testing the most
preferred hypotheses, until candidates are found; these candidates
are then minimal.
The ``preferred-last'' strategies generate and refine candidates until
their minimality is proven.
For each exploration strategy, we determine the conditions on the
hypothesis space that are necessary to ensure termination of the
diagnosis algorithm. 

Reiter's \texttt{diagnose} algorithm follows a preferred-first
strategy, but additionally uses \textit{conflicts} to improve its
efficiency. Conflicts enable the test solver to provide more
information when the outcome of a test is negative.
We generalise this idea and incorporate it into our preferred-first
strategy. In our framework, a conflict is a chunk of the hypothesis
space, which may be larger than the set of hypotheses tested, that
is proven to contain no candidate. We show that they can be represented
as sets of diagnostic properties that are inconsistent with the observed
system. Because at least one of these properties must be negated, 
conflicts focus the exploration of the hypothesis space 
and thus accelerate the search for a diagnosis. 

This work was motivated by our experience with real-world DES
diagnosis problems occurring in a number of application domains, including in
particular power systems alarm processing and business process
conformance checking, which we describe below. Existing model-based
diagnosis approaches were unable to cope with the complexity of these
problems. We use these problems to benchmark various instances of our
approach, differing in the hypotheses space, the strategy for
exploring it, and the test solver implementation chosen, against other
DES diagnosis methods. We show that our approach, using a test
solver based on SAT, is able to solve most of these problems, 
significantly outperforming earlier state-of-the-art algorithms.
We also obtain good performance with a test solver based on heuristic
search.

The present article builds on our earlier conference publications 
\cite{grastien-etal::dx::11,grastien:etal:12,grastien::ecai::14}.
The first article formulates diagnosis as a search problem on the hypothesis space 
and introduces the idea of a search strategy;
the second one explains how conflicts can be exploited 
for the specific case of diagnosis of discrete event systems; 
and the last one shows how the theory can be applied to hybrid systems.
Compared to these original works, 
we now present a unified theory motivated by a number of real world examples.
This theory is more thoroughly developed, complete with proofs, 
and comprehensively evaluated wrt other algorithms.

This paper is organised as follows:
In the next section, we provide some motivating examples 
for the present work.  
Section~\ref{sec::diag} gives a definition of the diagnosis problem 
that is independent from the modeling framework and the hypothesis space.  
Section~\ref{sec::sets} introduces the key concept of representation
of sets of hypotheses by sets of properties and explains how
questions relevant to diagnosis are formulated as diagnosis tests.
Section~\ref{sec::prop} demonstrates how these definitions are
instantiated for two different modeling frameworks: diagnosis of
circuits and diagnosis of discrete event systems.
Section~\ref{sec::strategies} presents 
different strategies for exploring the hypothesis space.
In Section~\ref{sec::related}, we discuss the relation to previous
work, in particular that which our theory generalises.
Section~\ref{sec::imple} describes two implementations of test solvers
for discrete event systems diagnosis, and Section~\ref{sec::xp} the
results of our experiments with these implementations.
Section~\ref{sec::future} concludes.


\section{Motivating Examples}
In this section, we briefly present examples of diagnosis problems
for discrete event and hybrid dynamical systems.
Each one of these problems requires a more expressive concept of
diagnosis than the classical ``set of faults'' definition, and thus
serves to motivate our general framing of the problem and our
generalisation of test-based diagnosis algorithms.

\subsection{Conformance Checking and Data Cleaning}
\label{sub::motiv::conformancechecking}

Deciding if a record of events matches or does not match a specified
process, or obeys or does not obey a set of rules, is a problem that
arises in several contexts.
It is known as \emph{conformance} or \emph{complicance checking} in
the Business Process Modelling (BPM) literature
\cite{van-der-aalst::survey::2013,hashmi:etal::KAIS::2018}.
Although there are many BPM formalisms, most of them model discrete
event systems.
Conformance checking may be just deciding whether the recorded event trace
matches the process specification (in diagnosis terms, whether the
system's execution is normal or abnormal), but often one seeks to find
a best \emph{trace alignment} \cite{degiacomo-etal:icaps:2016}: a set of
insertions (events missing from the trace), deletions (spurious events
in the trace) and substitutions (erroneous events in the trace) that
together are sufficient to make the event trace match the process.
In diagnosis terms, these adjustments to the trace are fault events,
and a best trace alignment corresponds to a minimal diagnosis candidate.
Note that in such a candidate, the same fault event may occur multiple
times, for example if the trace has multiple spurious events of the same
type. Thus, the space of diagnosis hypotheses can not be modelled simply
as sets of fault events.
The problem of process model adaptation, which examines event
traces corresponding to multiple executions of the process and seeks
a minimal modification of the process specification that suffices to
make all traces match, can likewise be viewed as an instance of DES
diagnosis.

Another example of diagnosis of recorded event traces occurs in
longitudinal, or temporal, databases, where each record denotes a
change in the status of some entity occurring at some time. The
ordered set of records relating to one entity forms a timeline, or
event trace, of that entity.
In the case study described by
Boselli et al.\ \cite{boselli-etal::icaps::14},
each entity is a person, and each
record pertains to a change in their employment status: starting
work for a new employer, ceasing work, extending a fixed-term
position or converting a current job between part-time and full-time,
or between fixed-term and continuing.
Entity timelines are typically subject to integrity constraints,
rules that prescribe events that cannot or must happen. For example,
a person must have started work with an employer before that job
can cease, or be converted or extended; a person can only hold one
full-time job at a time, and thus cannot start a part-time job if
already on a full-time position, or start a full-time job if already
working at all, but a person can start a new part-time job if they
are already working another part time.

However, errors and omissions in data entry mean that entity timelines
often do not satisfy the database rules. Rather than rejecting such
records, the problem of \emph{data cleaning} is to find a minimal set
of corrections that will restore consistency to the timeline
\cite{dallachiesa-etal:2013,geerts-etal:2013,boselli-etal::icaps::14}.
For example, consider the following timeline, from Boselli et al.'s
data set:
\begin{center}
  {\small
    \begin{tabular}{rllccl}
      Date & Worker & Event type & Full/Part & Term/Cont. & Employer \\
      \hline
      $d_1$    & 1370 & start   & full-time & fixed-term & 8274 \\
      $d_2$    & 1370 & cease   & full-time & fixed-term & 8274 \\
      \textit{$d_3$} & \textit{1370} & \textit{convert} & \textit{full-time} & \textit{fixed-term} & \textit{8274} \\
      \textit{$d_4$} & \textit{1370} & \textit{cease}   & \textit{full-time} & \textit{fixed-term} & \textit{8274} \\
      $d_5$ & 1370 & start   & full-time & fixed-term & 36638
    \end{tabular}
  }
\end{center}
The records on dates $d_3$ and $d_4$ violate the integrity constraints,
because they record a conversion event for a position that has already
ceased, and a double cessation record for the same position.

Like trace alignment, these corrections may be insertion of missing
records, deletion of spurious records, or changes to individual fields
of a record, including changes to the timing of records, and thus the
order of event in the timeline, and like in that case each correction
can occur multiple times in a timeline. Thus, viewed as a DES diagnosis
problem, a minimal diagnosis candidate is a multiset of faults events.
In the example above, the minimal diagnosis candidates include replacing
the conversion on $d_3$ with a ``start'' event (i.e., the person starting
work again for the same employer), or deleting the cessation event on
$d_2$ and changing either full-time or fixed-term status in the records
on $d_3$ and $d_4$.
Because there is no way to know with certainty which diagnosis candidate
corresponds to the true sequence of events, a data cleaning diagnoser
needs to return the complete set of minimal fault event multisets, for
a human to decide which corrections to apply or whether to investigate
further.

Note that when the diagnostic hypotheses are multisets (or sequences)
of faults rather than simple sets, the hypothesis space is infinite,
and even the set of candidates or the diagnosis may be infinite.
Close attention must therefore be given to avoiding non-termination
of the diagnosis algorithm.
In this paper, we present a number of algorithms that are able
to compute the complete diagnosis, also in infinite hypothesis
spaces, along with sufficient assumptions to guarantee their
termination (Section \ref{sec::strategies}).

\subsection{Alarm Processing}

In large complex systems, such as power grids or telecommunication networks, 
faults can produce non-trivial effects.
Alarms are time-stamped system-generated messages intended to aid
operators in diagnosing fault conditions and take timely corrective
actions. However, system complexity and the local nature of alarm
conditions mean that when a fault occurs, its secondary effects often
result in ``alarm cascades'' which obscure rather than inform about
the root cause. This problem has been recognised for some time
\cite{prince-wollenberg-bertagnolli89}, and there have been several
attempts to use AI techniques to ease the interpretation of alarms
through filtering, prioritising and explaining them
\cite{cordier:etal:98,cordier-dousson::safeprocess::00,taisne:06,larsson09,bauer-etal::dx::11}.
Framing the problem as dynamical system diagnosis treating unexplained
alarms as fault events means that a diagnoser can identify secondary
alarms, and thus focus attention on root causes
\cite{bauer-etal::dx::11,haslum-grastien::spark::11}.

Alarm logs have an important temporal dimension. For example, in a
power network, the event of a circuit breaker opening can explain
a following voltage drop alarm on the power line protected by the
breaker, if the breaker opening isolates the line. This implies
that the \emph{sequence} of fault (unexplained) events in the
diagnosis also matters: An unexplained circuit breaker opening
followed by an unexplained voltage drop does not carry the same
meaning as the same two unexplained alarms in the opposite order
(the former implies that it could not be inferred from the model
and observation that the breaker opening was sufficient to isolate
the line). Thus, the diagnostic hypotheses in this setting are
sequences of fault events, rather than sets.

Sequences of faults pose particular problems for classical
diagnosis algorithms.
Decomposition, for instance, is no longer as easy:
In the simple case when diagnostic hypotheses are sets of faults,
inferring independently that faults $f_1$ and $f_2$ are present
implies that any candidate fault set must contain $\set{f_1,f_2}$ as
a subset. However, when diagnostic hypotheses are fault sequences,
inferring the presence of fault events $f_1$ and $f_2$ does not
distinguish between sequences in which $f_1$ occurs before $f_2$
and those with the two events in opposite order.
Existing conflict-directed algorithms for diagnosis over fault-set
hypotheses are based on such a decomposition.
We show in this paper how the notion of conflict can be generalised
to any type of diagnostic hypothesis space. This is done by making
the concept of \emph{properties} of a hypothesis explicit, and
defining a set of properties that is sufficient to represent every
relevant hypothesis set, for any hypothesis space
(Section \ref{sub::properties}).

\subsection{Diagnosis of Hybrid Systems}

Hybrid systems are a class of models of dynamic systems that exhibit
both discrete mode changes and continuous evolution. Hybrid systems
naturally model physical processes under discrete control, such as
electrical systems
\cite{kurtoglu-etal:adapt:2009,fox-long-magazzeni:JAIR:2012}
and heating, ventilation, and air conditioning (HVAC) systems
\cite{behrens:provan:10,ono:etal:12,lim-etal::AAAI::2015}.
Diagnosis of hybrid systems can exhibit all the complexities of
discrete event systems, and more.
Consider, for example, the possible fault modes of sensors in the
\textsc{Adapt} benchmark system \cite{kurtoglu-etal:adapt:2009}:
When operating normally, a sensor's output is the real-valued sensed
value plus a bounded random noise. However, the sensor can fail by
becoming stuck at a fixed reading, by returning a value at a fixed
offset from the true reading, or by becoming subject to drift, which
is an offset value that increases over time. At a discrete abstraction
level, this is simply four possible fault modes, but a fully precise
diagnosis should also identify the offset constant or drift rate for
those fault modes.

The consistency-based approach can be applied to diagnosis of
hybrid systems \cite{grastien::ecai::14}, and has some advantages
over approaches that require a predictive model to simulate the system,
which are unable to handle unspecified or unpredictable behaviour
modes \cite{hofbaur-williams::tsmc::04}.
However, as we will show in this paper, there are limitations to
what can be guaranteed. If the diagnosis is required to estimate
real-valued fault parameters, such as the offset or drift rate of
a faulty sensor, the hypothesis space is \emph{dense}, in which case
a finite minimal diagnosis may not exist.


\section{The Diagnosis Problem}
\label{sec::diag}
In this section, we first present a generic definition of the
diagnosis problem, based on the notion of hypothesis space. 
The hypothesis space is motivated by the fact 
that different diagnostic environments 
(static systems and dynamic systems, in particular) 
require different types of 
diagnoses.
We then illustrate 
the generic definition with
different types of hypothesis spaces 
and discuss their relative expressiveness.  
Finally, we discuss a number of properties of these spaces 
that will influence the type of strategy that may be used
to explore the space.

\subsection{Diagnosis Definition}

We consider a system with a model $\MOD$, 
i.e., a description of all behaviours the system can exhibit.  
We assume this model is ``complete'', 
by which we mean that 
if a behaviour $\beh$ is possible in the system, 
then this behaviour is allowed by the model; 
we then write $\beh \in \MOD$.
A (partial) observation $o$ of the system is a predicate on behaviours:
$\OBS(\beh)$ is true if behaviour $\beh$ is consistent with
what has been observed.
We make no assumptions about how $\MOD$ and $\OBS$ are represented 
other than that they are of a form that the test solver (that is,
the theorem prover, model checker, etc, 
that will be used to reason about the system) can work
with. Typically, they will be given in some compact form (such as
a set of logical constraints, 
a factored representation of a discrete event system, or similar).

Given the model $\MOD$ and an observation $\OBS$,
the purpose of diagnosis is 
not to retrieve
the exact behaviour (or set of possible behaviours),
but to infer the diagnostic information associated with it.
For instance, we may want to identify which faults have occurred 
in the system and in which order.  
The diagnostic abstraction of a behaviour is called a ``hypothesis'' 
and we write $\hypo(\beh)$ for the (single) hypothesis 
associated with behaviour~$\beh$.
We write $\hypospace$ for the hypothesis space 
and we assume that hypotheses are mutually exclusive 
(i.e., $\hypo: \MOD \rightarrow \hypospace$ is a function).  
Because the system is only partially observable 
and may not be diagnosable\footnote{Diagnosability is the property that a fault 
will always be precisely identified; 
there is generally a correlation between diagnosability 
and the uniqueness of the diagnosis candidate~\cite{grastien-torta::sara::11}.}, 
it is generally not possible to precisely 
retrieve the hypothesis $\hypo(\beh)$.
Instead, the \textit{diagnosis} is the collection of hypotheses 
that are consistent (compatible) with both 
the model and the observation; 
such hypotheses are called ``diagnosis candidates''.  
From now on, we will use $\delta$ to represent a candidate, 
whilst $h$ will refer to a hypothesis that may not be a candidate.  

\begin{defi}[Diagnosis]
  Given a model $\MOD$, an observation $\OBS$, 
  and a hypothesis space $\hypospace$, 
  the \emph{diagnosis} is the subset $\Delta(\MOD,\OBS,\hypospace)$ 
  of hypotheses supported by at least one behaviour 
  consistent with the observation: 
  \begin{equation}
    \Delta(\MOD,\OBS,\hypospace) = 
    \{\delta \in \hypospace \mid 
    \exists \beh \in \MOD: \OBS(\beh) \land \hypo(\beh) = \delta
    \}.
  \end{equation}
\end{defi}

Because it asks only for consistency between the candidate and the
observation, this definition of diagnosis is weaker than 
that of an abductive diagnosis \cite{brusoni-etal::aij::98}, which 
requires each candidate to 
logically imply (part of) the observation. 

To make the diagnosis more
precise, it is common to impose a minimality condition.
The hypothesis space is equipped with a partial order relation $\preceq$ 
such that if $\delta \preceq \delta'$, 
then $\delta$ is preferred to $\delta'$, 
meaning $\delta'$ may be removed from the diagnosis.  
Recall that a partial order relation is antisymmetric, 
i.e., $\left(h \preceq h'\right) \land \left(h' \preceq h\right) 
\Rightarrow \left( h=h'\right)$.  
In the rest of the paper, we assume without loss of generality the existence of a unique most preferred hypothesis $h_0$ of $\hypospace$. This will normally correspond to the nominal system behavior, but if necessary (e.g., if they were multiple such behaviors or even an infinite number of them), one can always take $h_0$ to be a dummy hypothesis inconsistent with the system model.

We want to ignore the candidates 
that are not minimal with respect to $\preceq$,
where the subset of minimal elements of a set H is defined as
$\min_\preceq H = \{ h \in H \mid \nexists h' \in H.\ h' \prec h \}$.
We also want every ignored candidate to be supported by at least one
minimal candidate.
We then say that the minimal diagnosis \emph{covers} the diagnosis.  
Formally given two subsets $H$ and $H'$ of $\hypospace$, 
$H$ covers $H'$ 
if
\begin{displaymath}
  \forall h' \in H'.\ \exists h \in H.\ h \preceq h'.  
\end{displaymath}
The definition of the minimal diagnosis is as follows:

\begin{defi}[Minimal Diagnosis]\label{def:minimal-diagnosis}
  Given a model $\MOD$, an observation $\OBS$, 
  and a hypothesis space $\hypospace$,
  the subset $\Delta_\preceq(\MOD,\OBS,\hypospace)$ 
  of candidates in $\Delta(\MOD,\OBS,\hypospace)$
  that are minimal with respect to $\preceq$ is
  the \emph{minimal diagnosis} if it covers the diagnosis.  
\end{defi}

In most diagnosis environments, it will be the case that a minimal
diagnosis always exists. However, in Subsection
\ref{sec::hypo-space-props} we show an example of where it does not.  
To simplify notation, we will in the rest of the paper omit the
parameters from the diagnosis and the minimal diagnosis, i.e., we
will simply write $\Delta$ and $\Delta_\preceq$. 

\subsection{Examples of Hypothesis Spaces}
\label{sub::examplesofspaces}

The simplest hypothesis space is the Binary Hypothesis Space (BHS), 
where each behaviour is classified only as either nominal or faulty.
This leads to a fault detection problem rather than a diagnosis one. 
The preferred hypothesis is generally the nominal hypothesis.  

The most commonly used hypothesis space is the Set Hypothesis Space (SHS).  
Given a set $F$ of faults, 
a hypothesis is the subset $h \subseteq F$ of faults 
that appear in the behaviour.  
Preference is given to hypotheses 
that contain a subset of faults: 
$h \preceq h' \Leftrightarrow h \subseteq h'$.  

Another popular hypothesis is the Minimal Cardinality Set Hypothesis Space (MC-SHS).
Hypotheses are defined similarly to SHS,
as the set of faults that affect the system.
The preference relation however is defined through the number of faults,
with $h$ preferred over $h'$
if it has the smallest cardinality (number of faults in the hypothesis):
\begin{displaymath}
  h \preceq h' \Leftrightarrow \bigg( h = h'\ \lor\ |h| < |h'| \bigg).
\end{displaymath}

For the case where the probability of faults varies,
each fault $f$ is associated with an a-priori probability $Pr(f) \in (0,0.5)$,
and the a-priori probability of hypothesis $h$ is then
$h = \Pi_{f \in h}\ Pr(f) \times \Pi_{f \in F \setminus h}\ (1-Pr(f))$.
The preference relation of the A-priori Probability Set Hypothesis Space (AP-SHS)
then maximises the a-priori probability:
\begin{displaymath}
  h \preceq h' \Leftrightarrow \bigg( h = h' \ \lor\ Pr(h) < Pr(h')\bigg).
\end{displaymath}
Bylander et al. proposed more elaborate definitions 
based on qualitative plausibilities of the hypotheses 
\cite{bylander-etal::aij::91}.

Our theory does not handle diagnosis problems
in which probability is maximised a-posteriori,
i.e., after the likelihood of the hypothesis given the observations has been factored in
\cite{LUCAS200199}.

In dynamic systems, faults may occur several times.  
The Multiset Hypothesis Space (MHS) associates each fault 
with the number of occurrences of this fault: 
$h: F \rightarrow \mathbf{N}$.  
A hypothesis is preferred to another
if it has no more occurrences of any fault: 
$h \preceq h' \Leftrightarrow 
\left(\forall f \in F,\ h(f) \le h'(f)\right)$.  

If we wish to also distinguish the order of occurrences of faults,
a hypothesis in the Sequence Hypothesis Space (SqHS) is a (possibly
empty) sequence of faults: $h \in F^\star$.  
A hypothesis is preferred to another
if the former is a subsequence of the latter.  
Formally, if $h = [f_1,\dots,f_k]$ and $h' = [f'_1,\dots,f'_n]$, 
then $h \preceq h' \Leftrightarrow 
\exists g: \{1,\dots,k\} \rightarrow \{1,\dots,n\}:$ 
$\big(\forall i \in \{1,\dots,k-1\},\ g(i) < g(i+1)\big)$
$\land$
$\big(\forall i \in \{1,\dots,k\},\ f_i = f'_{g(i)}\big)$.  
For instance, hypothesis $[a,b]$ is preferable 
to hypothesis $[c,a,d,b]$.  

We can also strengthen the preference order to treat faults differently,
for instance, to reflect their relative likelihood.
As an example, we consider the Ordered Multiset Hypothesis Space (OMHS).
The hypotheses in this space are the same as in MHS, i.e., mappings
from each fault to the number of times it occurred, but we also have
an ordering of the faults, and any number of occurrences of a fault $f'$
is preferred to a single occurrence of a fault $f \prec f'$.
Formally, $h \preceq h' \Leftrightarrow 
\forall f' \in F,\ h(f') > h'(f') \Rightarrow 
\exists f \in F: \big( f \prec f' \big) \land \big( h(f) < h'(f) \big)$.
This corresponds to fault $f'$ being infinitely more likely than fault $f$.

Finally, we consider faults that are represented by a continuous value.
This can be used to model, for example, the situation where the fault is
a drift in a model parameter.
We assume a single continuous-valued fault. This is a very simple case,
but it will be sufficient for illustrative purposes. 
In the Continuous Hypothesis Space (CHS), 
a hypothesis is a positive real value: $h \in \mathbf{R}^+$.  
Preference is given to smaller values: 
$h \preceq h' \Leftrightarrow h \le h'$.

\subsection{Properties of Hypothesis Spaces}
\label{sec::hypo-space-props}

In this section, we define 
the terminology related to hypothesis spaces which will be used to define our 
framework and 
to formulate termination conditions for different exploration
strategies.

\paragraph{{\bf Relations Between Hypotheses}}
If $h \preceq h'$, we say that $h$ is an \textit{ancestor} of $h'$ 
and that $h'$ is a \textit{descendant} of $h$ 
(note that, since $\preceq$ is non-strict, $h$ is an ancestor and a
descendant of itself).  
If $h \prec h'$ and there is no $h''$ such that $h \prec h'' \prec h'$, 
then we say that $h'$ is a \textit{child} of $h$ 
and $h$ is a \textit{parent} of $h'$.  

\paragraph{{\bf Finiteness}}
The first condition we consider
is whether the hypothesis space is finite.
Infinite hypothesis spaces must be dealt with more cautiously,
as they may prevent the diagnosis algorithm from terminating.
In a finite space, any systematic exploration strategy (i.e, one that
does not revisit a previously rejected hypothesis) will terminate.
BHS, SHS, MC-SHS, and AP-SHS are finite.  

\paragraph{{\bf Well Partial Orderness}}
A binary relation on a set $\setspace$ is a \textit{well partial order} iff
it is a (non-strict) partial order and every non-empty subset of $\setspace$
has a finite and non-empty set of minimal elements according to the
order (e.g., \cite{kruskal::joct::72}). That is,
\begin{displaymath}
  \forall S \subseteq \setspace.\quad
  S \neq \emptyset\ \Rightarrow\ 0 < | \min_{\preceq}(S) | < \infty.  
\end{displaymath}
If the preference order $\preceq$ is a well partial order on
$\hypospace$, we say that $\hypospace$ is \textit{well partially ordered}
(by $\preceq$).
A well partial order is always well-founded, meaning it has no
infinite descending chains.

The continuous hypothesis space given in the previous section
(CHS) is not well partially ordered.
To see this, consider the set of hypotheses that correspond to a
strictly positive value, i.e.,
$S = \{h \in \hypospace_{\mathrm{CHS}} \mid h > 0\}$. This set has no
minimal value, which means that $\min_{\preceq}(S)$ is empty. 
All the other hypothesis spaces discussed in the previous section are
well partially ordered. For the non-trivial cases of MHS and SqHS,
this follows from the work of Nash-Williams \cite{nashwilliams::mpcps::63}
on well-quasi-ordered finite trees.

Well partially ordered hypothesis spaces have several useful properties:
First, that the minimal diagnosis always exists and is finite (this is
shown in Theorem~\ref{theo::minimaldiagnosis} below). Second, that the
set of parents and the set of children of any given hypothesis are both
finite. This follows from 
the fact that all parents of a hypothesis are themselves
unordered; thus, they are all minimal in the set of the hypothesis'
parents and, therefore, there cannot be infinitely many of them. The same is true of its children.
Third, any strict descendant of a hypothesis 
is also a (possibly non-strict) descendant 
of some child of that hypothesis.  

\begin{theo}\label{theo::minimaldiagnosis}
  If the hypothesis space is well partially ordered,
  then the minimal diagnosis exists and is defined by: 
  \begin{equation}
    \Delta_\preceq = \min_\preceq(\Delta)
    = \{h \in \Delta \mid \forall h' \in \Delta,\ 
    h' \preceq h \Rightarrow h = h'\}.
  \end{equation}
  Furthermore, $\Delta_\preceq$ is finite.
\end{theo}

 \begin{proof}
  We must show that $\min_\preceq(\Delta)$ satisfies 
  the condition of Definition~\ref{def:minimal-diagnosis} 
  which states that $\min_\preceq(\Delta)$ must cover the diagnosis.  

  Assume that the diagnosis is not covered by $\min_\preceq(\Delta)$.  
  Let $\delta_1$ be a diagnosis candidate that is not covered: 
  $\nexists \delta' \in \min_\preceq(\Delta)$ 
  such that $\delta' \preceq \delta_1$.  
  Then because $\delta_1 \not\in \min_\preceq(\Delta)$, 
  there exists another preferable candidate $\delta_2 \prec \delta_1$ 
  that is not covered.  
  Applying the same reasoning, we end up 
  with an infinite sequence of hypotheses 
  $\delta_1 \succ \delta_2 \succ \dots$  
  This sequence contradicts the property of well partially order.  
\end{proof}

If the space is not well partially ordered, there is no such
guarantee. For instance, in the CHS, if
$\Delta =\{h \in \hypospace_{\mathrm{CHS}} \mid h > 0\}$, as in the
example above, then $\min_{\preceq}(\Delta) = \emptyset$ which does
not satisfy the covering requirement 
of Definition~\ref{def:minimal-diagnosis}.
Thus, in this situation there exists no minimal diagnosis.

\paragraph{{\bf Path, Depth and Distance}}
Finally, we define concepts that relate 
to a hypothesis' ``position'' in the hypothesis space, 
which we will use in Section~\ref{sec::strategies} 
when proving termination of our diagnosis algorithms.  

A \emph{path} from hypothesis $h$ to $h'$ 
is a sequence of hypotheses $h_1\prec \dots \prec h_k$ 
such that $h_1 = h$ and $h' = h_k$.  
An \emph{atomic path} is a path $h_1\prec\dots\prec h_k$ 
such that each $h_i$ is a parent of $h_{i+1}$.

The \emph{distance} of hypothesis $h$ (implicitely from $h_0$) 
is the minimal length of an atomic path 
from hypothesis $h_0$ to hypothesis $h$; 
if no such atomic path exists, 
hypothesis $h$ is said to have an infinite distance.  
A hypothesis is said to be \emph{finitely reachable} 
if it has a finite distance.  

The ordered multiset hypothesis space (OMHS) 
illustrates a situation with non-finitely reachable hypotheses.  
Assume two fault events, $f_1$ and $f_2$ where $f_1 \prec f_2$ 
(any number of occurrences of $f_2$ 
is preferred to one occurrence of $f_1$) 
and consider hypothesis $h = \{ f_1 \rightarrow 1, f_2 \rightarrow 0\}$.  
Then $h$ has no parent: 
indeed, all strict ancestors of $h$ are hypotheses $h_i$
with no occurrence of $f_1$ and $i$ occurrences of $f_2$:
$h_i = \{ f_1 \rightarrow 0, f_2 \rightarrow i\}$.
Then for all $i$ the property $h_i \prec h_{i+1} \prec h$ holds,
and $h_i$ is not a parent of $h$.  
Since 
$h$ has no parent 
no atomic path leads to $h$.  

The \emph{depth} of a hypothesis $h$ 
is the maximal length of a path from $h_0$ to $h$.  
If there is no maximal length, 
the depth is said to be infinite.  

The depth of a hypothesis is, by definition, 
larger than or equal to its distance, hence
a hypothesis that is not finitely-reachable 
has an infinite depth.  
The converse may not hold however: 
there can be a finite atomic path $h_0 \prec h_1 \prec h$ 
and, at the same time, an infinite number of paths 
$h_0 \prec h'_1 \prec \dots \prec h'_k \prec h$ 
for any $k$.
To find an example we have to look 
at some even more fine-grained preference order. 
For example, with reference to Figure~\ref{fig::distancedepth}, consider a system
consisting of a component monitored by a sensor.  
The component can exhibit any number of temporary failures 
(represented by a natural number), 
while the sensor has two modes: nominal ($N$) and faulty ($F$).  
It is assumed that 
the component  and the sensor both 
experiencing faults is infinitely more unlikely 
than any number of faults on the component.  
Consequently, $h_0 = \langle 0,N\rangle$ 
is the unique preferred hypothesis; 
$h_1 = \langle 0,F\rangle$ is a child of $h_0$ 
(any $h'_i = \langle i,N\rangle$, $i \ge 1$, 
is incomparable to $h_1$); 
$h = \langle 1,F\rangle$ is a child of $h_1$ 
(there is no hypothesis $h'$ such that $h_1 \prec h' \prec h$) 
hence $h$'s distance is $2$ and $h$ is finitely-reachable.  
On the other hand, we have $h_0 \prec h'_1 \prec h'_2 \prec 
\dots \prec h$, i.e., $h$ is infinitely deep.  

\begin{figure}[ht]
  \begin{center}
    \begin{tikzpicture}
  \node (0n) at (2,6) {$\langle 0,N\rangle$};

  \node (1n) at (0,5) {$\langle 1,N\rangle$};
  \node (2n) at (0,4) {$\langle 2,N\rangle$};
  \node (dn) at (0,3) {$\dots$};
  \node (in) at (0,2) {$\langle i,N\rangle$};
  \node (en) at (0,1) {$\dots$};

  \node (0f) at (4,3) {$\langle 0,F\rangle$};

  \node (1f) at (2,0) {$\langle 1,F\rangle$};

  \draw[->] (1f) -- (0f);
  \draw[->] (0f) -- (0n);
  \draw[->,dashed] (1f) -- (en);
  \draw[->,dashed] (1f) -- (in);
  \draw[->,dashed] (1f) -- (dn);
  \draw[->,dashed] (1f) -- (2n);
  \draw[->,dashed] (1f) -- (1n);
  \draw[->] (en) -- (in);
  \draw[->] (in) -- (dn);
  \draw[->] (dn) -- (2n);
  \draw[->] (2n) -- (1n);
  \draw[->] (1n) -- (0n);
\end{tikzpicture}
  \end{center}
  \caption{Hypothesis space illustrating that the depth can be infinite
    while the distance is finite.
    An unbroken line indicates a parent/child relationship;
    a dashed one, an ancestor one.
    The distance between $\langle 0,N\rangle$ and $\langle 1,F\rangle$ is two;
    the depth is infinite.
  }
  \label{fig::distancedepth}
\end{figure}
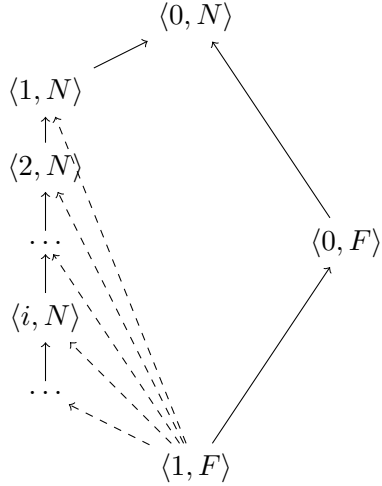

\subsection{Abstraction of Hypothesis Spaces}
\label{sub::expressiveness}

In the previous section, we hinted
that the diagnosis in some hypothesis spaces 
is more informative than in others.  
We now formalise this notion.  

A hypothesis space $\hypospace$ 
(together with its preference relation $\preceq$ 
and its function $\hypo: \MOD \rightarrow \hypospace$) 
is a \emph{refinement} of hypothesis space $\hypospace'$ 
(together with $\preceq'$ and $\hypo'$), 
and conversely $\hypospace'$ is an \emph{abstraction} of $\hypospace$, 
if each hypothesis of $\hypospace'$ 
corresponds exactly to a subset of hypotheses in $\hypospace$.  
Formally, there exists a function 
$\projhypof: \hypospace \rightarrow \hypospace'$ 
that projects each hypothesis of $\hypospace$ on $\hypospace'$ 
such that
\begin{itemize}
\item 
  $\forall \beh \in \MOD.\ \hypo'(\beh) = \projhypo{\hypo(\beh)}$, 
  i.e., $\hypospace'$ is an abstraction of $\hypospace$, 
  and
\item 
  $\forall \{h_1,h_2\} \subseteq \hypospace.\ 
  h_1 \preceq h_2 \Rightarrow \projhypo{h_1} \preceq' \projhypo{h_2}$, 
  i.e., the preference relation is maintained by the abstraction. 
\end{itemize}
The projection is extended naturally to a set of hypotheses, 
i.e., $\projhypo{H} = \{ h' \in \hypospace' \mid 
\exists h \in H.\ h' = \projhypo{h}\}$.  

For instance, the set hypothesis space is an abstraction of the
multiset hypothesis space (over the same set of faults).
Given a multiset hypothesis, 
i.e., a mapping $F \rightarrow \mathbf{N}$, 
the abstraction function $\projhypof$ 
returns the subset of faults that are associated 
with a strictly positive number: 
$\projhypo{h} = \{f \in F \mid h(f) > 0\}$.  
Furthermore, the preference relation is maintained: 
if $h_1 \preceq_{\mathrm{MHS}} h_2$, 
then $h_1(f) \le h_2(f)$ for all $f$; 
consequently, $\projhypo{h_1} \subseteq \projhypo{h_2}$ 
and $\projhypo{h_1} \preceq_{\mathrm{SHS}} \projhypo{h_2}$.  

An abstraction/refinement relationship between two hypothesis
spaces implies that the diagnoses (and minimal diagnoses) in those
two spaces are also related. This is shown by the following two
lemmas. Theorem~\ref{theo::all-abstractions} below states all
abstraction relations (summarised in Figure~\ref{fig::abtractions})
between the hypothesis spaces for discrete event systems 
(BHS, SHS, MC-SHS, AP-SHS, OMHS, MHS, and SqHS) 
described in the previous subsection.

\begin{lemm}\label{lemm::abstractionofdiag}
  If $\hypospace'$ is an abstraction of $\hypospace$, 
  the projection on $\hypospace'$ of the diagnosis in $\hypospace$ 
  is the diagnosis in $\hypospace'$: 
  $\projhypo{\Delta} = \Delta'$.  
\end{lemm}

\begin{proof}
  We prove that $\Delta'$ is exactly the set of hypotheses 
  $\delta' = \projhypo{\delta}$ for some candidate $\delta \in \Delta$.  
  \begin{displaymath}
    \begin{array}{c @{\quad\Rightarrow\quad} c}
      \delta \in \Delta & 
      \exists \beh \in \MOD.\ \OBS(\beh) \land \hypo(\beh) = \delta\\
      & 
      \exists \beh \in \MOD.\ \OBS(\beh) \land 
      \hypo'(\beh) = \projhypo{\delta}\\
      & 
      \projhypo{\delta} \in \Delta'
    \end{array}
  \end{displaymath}
  Conversely, 
  \begin{displaymath}
    \begin{array}{c @{\quad\Rightarrow\quad} c}
      \delta' \in \Delta' &
      \exists \beh \in \MOD.\ \OBS(\beh) \land 
      \hypo'(\beh) = \delta'\\
      & 
      \exists \beh \in \MOD.\ 
      \hypo(\beh) \in \Delta \land \hypo'(\beh) = \delta'\\
      & 
      \exists \beh \in \MOD.\ 
      \hypo(\beh) \in \Delta \land \projhypo{\hypo(\beh)} = \delta'\\
      & 
      \exists \delta \in \Delta.\ 
      \projhypo{\delta} = \delta'
    \end{array}
  \end{displaymath}
\end{proof}

\begin{lemm}
  If $\hypospace'$ is an abstraction of $\hypospace$, 
  the projection on $\hypospace'$ of the minimal diagnosis in $\hypospace$ 
  is contained in the diagnosis in $\hypospace'$ 
  and contains the minimal diagnosis in $\hypospace'$: 
  $\Delta'_{\preceq'} \subseteq 
  \projhypo{\Delta_{\preceq}}
  \subseteq \Delta'$.  
\end{lemm}

\begin{proof}
  Since $\Delta_{\preceq} \subseteq \Delta$ then clearly
  $\projhypo{\Delta_{\preceq}} \subseteq \projhypo{\Delta} = \Delta'$.  

  Assume now that there exists a minimal candidate $\delta'_1$ in
  $\hypospace'$ such that $\delta'_1 \in \Delta'_{\preceq'} 
  \setminus \projhypo{\Delta_{\preceq}}$.  
  Then, by Lemma~\ref{lemm::abstractionofdiag}, 
  there exists a candidate $\delta_1 \in \Delta$ 
  such that $\projhypo{\delta_1} = \delta'_1$.  
  Furthermore, since $\delta'_1 \not\in \projhypo{\Delta_{\preceq}}$, 
  $\delta_1 \not\in \Delta_{\preceq}$.  
  Therefore, there must exist another candidate 
  $\delta_2 \in \Delta_{\preceq}$ 
  such that i) $\delta_2 \preceq \delta_1$ 
  (which is why $\delta_1 \not\in \Delta_{\preceq}$) 
  and ii) $\projhypo{\delta_2} = \delta'_2 \neq \delta'_1$ 
  (since $\delta'_2 \in \projhypo{\Delta_{\preceq}}$ but
  $\delta'_1 \not\in \projhypo{\Delta_{\preceq}}$).  
  However, by Lemma~\ref{lemm::abstractionofdiag}, 
  $\delta'_2$ is a candidate, and by the second condition on $\projhypof$, 
  $\delta'_2 \preceq \delta'_1$.  
  Hence, $\delta'_1$ is not a minimal candidate, 
  which contradicts its existence.  
\end{proof}

In other words, the projection of the minimal diagnosis $\Delta_{\preceq}$
in $\hypospace$ is a subset of (possibly equal to) the diagnosis in the
more abstract space $\hypospace'$, whose minimisation is the minimal
diagnosis in $\hypospace'$.

Returning to the example of the set and multiset hypothesis spaces,
given a minimal diagnosis $\Delta^{\mathrm{MHS}}_{\preceq} = \{
\{a \rightarrow 2, b \rightarrow 0\},
\{a \rightarrow 1, b \rightarrow 1\}, 
\{a \rightarrow 0, b \rightarrow 2\}
\}$ in the multiset hypothesis space, 
its projection on the set hypothesis space is
$\projhypo{\Delta^{\mathrm{MHS}}_{\preceq}} = \{\{a\}, \{a,b\}, \{b\}\}$.
The minimal diagnosis in the set hypothesis space is
$\Delta^{\mathrm{SHS}}_{\preceq} = \{\{a\},\{b\}\}$, 
which is the set of minimal elements of 
$\projhypo{\Delta^{\mathrm{MHS}}_{\preceq}}$.

This relation between the (minimal) diagnosis in a hypothesis space
$\hypospace$ and an abstraction $\hypospace'$ of $\hypospace$ has
implications for the complexity of computing it:
Since the (minimal) diagnosis in $\hypospace'$ can be computed 
from the (minimal) diagnosis in $\hypospace$, in time polynomial
in the size of the diagnosis, we can say that diagnosing in a more
refined hypothesis space is at least as hard as diagnosing in the
more abstract space.

\begin{theo}\label{theo::all-abstractions}
  The set of abstraction relations between hypothesis spaces shown
  in Figure~\ref{fig::abtractions} is correct and complete.
\end{theo}

\begin{figure}[ht]
  \begin{center}
    \begin{tikzpicture}
\draw (0, 0) node (BHS)  {BHS};
\draw (2, 1.1)   node (CASHS)  {CA-SHS};
\draw (2, 0.5)   node (APSHS)  {AP-SHS};
\draw (4, 0.8)   node (SHS)  {SHS};
\draw (3, -0.5)   node (OMHS) {OMHS};
\draw (6, 0)   node (MHS)  {MHS};
\draw (9, 0)   node (SqHS) {SqHS};
\draw[<-] (BHS) -> (CASHS);
\draw[<-] (CASHS) -> (SHS);
\draw[<-] (BHS) -> (APSHS);
\draw[<-] (APSHS) -> (SHS);
\draw[<-] (SHS) -> (MHS);
\draw[<-] (MHS) -> (SqHS);
\draw[<-] (BHS) -> (OMHS);
\draw[<-] (OMHS) -> (MHS);
\end{tikzpicture}
  \end{center}
  \caption{Abstraction relations between the hypothesis spaces of DES 
    presented in Subsection~\ref{sub::examplesofspaces}; 
    $\hypospace'$ is an abstraction of $\hypospace$ iff there
    is a directed path from $\hypospace$ to $\hypospace'$.}
  \label{fig::abtractions}
\end{figure}
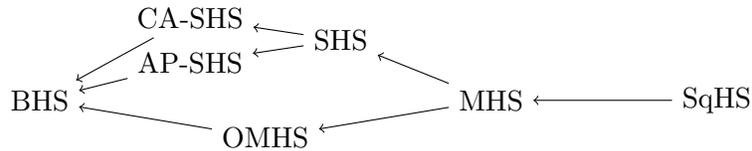

\begin{proof}(Sketch)
  {
  The abstraction function from *SHS to BHS 
  is $\projhypo{h} = \mathrm{nominal}$ iff $h = \emptyset$.
  The abstraction function from SHS to CA-SHS is the identity function,
  and the preference relation of SHS is indeed maintained in CA-SHS:
  $h \subseteq h' \Rightarrow \bigg(h = h' \ \lor \ |h| < |h'|\bigg)$.
  Similarly the preference between two SHS hypotheses is maintained
  when these hypotheses are interpreted as AP-SHS thanks to the fact
  that each fault has an a-priori probability below $0.5$,
  which implies that removing a fault from a hypothesis increases its a-priori probability.
  }
  The abstraction function from MHS to SHS has already been described.
  The abstraction function from SqHS to MHS counts the number of
  occurrences of each faulty event in the sequence.  
  The abstraction function from OMHS to BHS 
  is $\projhypo{h} = \mathrm{nominal}$ 
  iff $h(f) = 0$ for all faulty events $f$.  
  The abstraction function from MHS to OMHS is the identity function; 
  OMHS is an abstraction of MHS because its associated precedence relation 
  is more restrictive than that of MHS.  
 Finally the abstraction function from CHS to BHS 
 is $\projhypo{h} = \mathrm{nominal}$ iff $h = 0$.  

  There is no relation between SHS and OMHS 
  since SHS does not mention the number of occurrences as OMHS does, 
  while the mapping from OMHS to SHS does not maintain the preference
  relation: for instance, if $a \prec b$, then 
  $\{a \rightarrow 0, b \rightarrow 1\} 
  \prec_{\mathrm{OMHS}} \{a \rightarrow 1, b \rightarrow 0\}$, 
  while $\{b\} \not\prec_{\mathrm{SHS}} \{a\}$.  
\end{proof}


\section{Representing and Testing Sets of Hypotheses}
\label{sec::sets}
The diagnosis approach developed in this paper 
is based on an operation called the \textit{diagnosis test}.
A test, defined in Subsection~\ref{sub::diagnosistest}, decides 
whether a given set of hypotheses has a non-empty intersection
with the diagnosis, that is, whether any hypothesis in the set
is a candidate.
The set of hypotheses to be tested is not enumerated but represented
symbolically. To this end, we define in Subsection \ref{sub::properties}
\textit{hypothesis properties}, which are atomic statements used to
describe hypotheses. We show how to construct for any hypothesis space
a matching property space
that is ``sufficient'', in the sense that
any set of hypotheses that we need to test has a representation using
properties in this space. In Subsection \ref{sub::relevantsetsofproperties}
we discuss three specific types of tests, which we term ``diagnosis
questions'', that together are sufficient to implement the exploration
strategies we propose.
The strategies themselves are described in Section~\ref{sec::strategies}.

Here, and in the remainder of the paper, we consider only well
partially ordered hypothesis spaces. As shown earlier, this ensures
that the minimal diagnosis exists and is finite, so that the diagnosis
algorithm can output it in finite time. 

\subsection{The Diagnosis Test}\label{sub::diagnosistest}

Our diagnosis algorithms are based on an operation 
called the \textit{diagnosis test}. 
We assume the existence of an ``oracle'', called the \textit{test solver}, 
that is able to perform such tests.
We will describe several concrete implementations of test solvers,
for DES and different hypothesis spaces, in 
Section~\ref{sec::imple}.

A diagnosis test is the problem 
of deciding whether a given set $H \subseteq \hypospace$ 
contains a diagnosis candidate.  

\begin{defi}\label{defi::test}
  A \emph{diagnosis test} is a tuple $\langle \MOD, \OBS, H\rangle$ 
  where $\MOD$ is a system model, $o$ is an observation, 
  and $H \subseteq \hypospace$ is a set of hypotheses. 

  The \emph{result} of a test is either a hypothesis $\delta \in H$ 
  such that $\delta \in \Delta(\MOD,\OBS,\hypospace)$ if any such
  $\delta$ exists, 
  and $\bot$ otherwise (where $\bot \notin \hypospace$ 
  is a distinct symbol).
\end{defi}

Later, in Section \ref{sec::conflicts}, we will amend this definition
to allow the test solver to return a conflict instead of $\bot$,
but for now we limit ourselves to the simple version.  
Given a diagnosis problem $\langle \MOD,\OBS,\hypospace\rangle$, 
a test is defined solely by the hypothesis set $H$.  
If the test returns a candidate, we say it is successful; 
otherwise, we say it failed.  

\subsection{Hypothesis Properties}\label{sub::properties}

Some of the sets of hypotheses we will need to test to compute the
diagnosis can be very large, and some of them will even be infinite.
Therefore, we represent such sets symbolically, by a finite set of
\textit{hypothesis properties}. 
These properties are atomic statements about hypotheses. 
A set of properties represents those hypotheses 
that satisfy all properties in the set.

Not all sets of hypotheses will be represented in this way.
The minimal diagnosis returned by our algorithms is an explicitly
enumerated set of candidates, as are some other sets manipulated by
the algorithms during computation of the minimal diagnosis. However,
all hypothesis sets given to the test solver to test are represented
symbolically; that is, the test solver's input will be a set of
properties, rather than a set of hypotheses.
To distinguish the two types of sets, 
we will use $H$ for sets of hypotheses represented symbolically 
and $S$ for explicitly enumerated hypothesis sets.

\begin{defi}\label{defi::prop}
  A \emph{hypothesis property} (or simply, property) is an object $p$
  that implicitly represents a (possibly infinite) set of hypotheses
  $\hypos(p) \subseteq \hypospace$.  If hypothesis $h$ belongs to
  $\hypos(p)$, we say that $h$ exhibits property $p$, or that $p$ is a
  property of $h$. For any property~$p$, we also use $\neg p$ as a
  property, with the meaning $\hypos(\neg p)
  = \hypospace \setminus \hypos(p)$.  

  Given a hypothesis property space $\propspace$, we write
  $\props(h)\subseteq \propspace$ for the set of properties of $h$.
  A set $P\subseteq \propspace$ of properties implicitly represents 
  the set $\hypos(P)$ of hypotheses that exhibit all properties in $P$: 
  $\hypos(P) 
  = \{h \in \hypospace \mid P \subseteq \props(h)\} 
  = \bigcap_{p \in P} \hypos(p)$.  
\end{defi}

Simple examples of properties are that a given fault occurred; 
or did not; that it occurred at most once; or more than once;
that one type of fault occurred before another; and so on. 
We give more examples of properties later 
in this subsection.  

A priori, we can define properties to represent any set of
hypotheses. Given a set $H$ of hypotheses, we could define a
property $p_H$ such that $\hypos(p_H) = H$. 
However, implementing support for such ad hoc properties in the
test solver is not practical, and is also not very useful, since
it does not help in the formation of informative conflicts.
Useful properties are ones that allow the test solver to
automatically infer information that can be generalised.  
For instance, the property that states that a specific fault did
not occur is of this kind.

Next, we define the property space $\propspace$ that we will use in
the rest of this paper. $\propspace$ is derived from the hypothesis space 
$\hypospace$ considered and its preference relation, and is therefore
defined for any hypothesis space.
For each hypothesis $h \in \hypospace$, 
$\propspace$ contains the following two properties 
and their negations:
\begin{itemize}
\item 
  $\propdesc(h)$ is the property of being a descendant of hypothesis $h$,
  i.e., $\hypos(\propdesc(h)) = \{h' \in \hypospace \mid h \preceq h'\}$ 
  and 
\item 
  $\propanc(h)$ is the property of being an ancestor of hypothesis $h$,
  i.e., $\hypos(\propanc(h)) = \{h' \in \hypospace \mid h' \preceq h\}$.  
\end{itemize}

These properties may appear somewhat abstract; their concrete meaning
depends on the hypothesis space and preference order that underlies
them. To give a more concrete example, let us look at the set hypothesis
space (SHS): Let $h = \{f_1,f_2\} \subseteq F = \{f_1,\dots,f_4\}$ 
be the hypothesis that faults $f_1$ and $f_2$ took place, while the
other two faults ($f_3$ and $f_4$) did not.
Then 
\begin{itemize}
\item $\propdesc(h)$ is the property that $f_1$ and $f_2$ took place
 (not ruling out that other faults may also have happened);
\item $\neg\propdesc(h)$ is the property that not both $f_1$ and $f_2$
 occurred;
\item $\propanc(h)$ is the property that no fault other than $f_1$
 or $f_2$ took place, i.e., neither $f_3$ nor $f_4$; and
\item $\neg\propanc(h)$ is the property that some fault other than
 $f_1$ and $f_2$ took place, i.e., either $f_3$ or $f_4$ happened.
\end{itemize}

These properties are sufficient to represent all of the sets of
hypotheses that we will need to test in any of our strategies for
exploring the hypothesis space. In fact, we can give a more precise
characterisation of the hypothesis sets that can be represented with
conjunctions of properties in $\propspace$.
To do this, we first need to recall some standard terminology: Let
$\preceq$ be a partial order on some set $\setspace$; 
a subset $S$ of $\setspace$ is \textit{convex} 
iff for any two distinct elements $a,b \in S$,
every element $c$ such that $a \preceq c \preceq b$ is also in $S$.

\begin{theo}
Hypothesis set $H \subseteq \hypospace$ can be represented by a
finite conjunction of properties over $\propspace$ if and only if $H$
is convex.
\end{theo}

\begin{proof}
First, let $H$ be a convex hypothesis set.
If $H = \emptyset$, the claim holds trivially, since the empty set
can be represented by any contradictory set of properties, e.g.,
$\{\propdesc(h), \neg\propdesc(h)\}$.
Therefore, suppose $H$ is non-empty.

Let $H^{\prec} = \{ h' \not\in H \mid \exists h \in H : h' \preceq h \}$,
$H^{\succ} = \{ h' \not\in H \mid \exists h \in H : h \preceq h' \}$, and
$H^{\textrm{u}} = \{ h' \not\in H \mid \forall h \in H : h' \not\preceq h
 \, \textrm{and} \, h \not\preceq h' \}$, that is, $H^{\prec}$ is the set
of ancestors of hypotheses in $H$ that are not themselves in $H$,
$H^{\succ}$ is the set of descendants of hypotheses in $H$ that are not
themselves in $H$, and $H^{\textrm{u}}$ is the set of hypotheses that are
not ordered with respect to any element in $H$. Because $H$ is convex,
every hypothesis $h' \in \hypospace \setminus H$ must belong to one of
these three sets: if $h'$ is not unrelated to every hypothesis in $H$,
it must be either preferred to some $h \in H$, or some $h \in H$
preferred to it; thus it belongs to either $H^{\prec}$, $H^{\succ}$.
Furthermore, it cannot belong to both: if it did, there would be some
hypothesis $h \in H$ such that $h \preceq h'$ and some hypothesis
$h'' \in H$ such that $h' \preceq h''$; this contradicts the convexity
of $H$.

Construct the property set $P = \{ \neg\propanc(h') \mid h' \in
 \max_\preceq(H^\prec) \} \cup \{ \neg\propdesc(h') \mid h' \in
 \min_\preceq(H^\succ ) \} \cup \{ \neg\propdesc(h') \mid h' \in
 \min_\preceq(H^{\textrm{u}}) \}$. We claim that $P$ is finite and
that $\hypos(P) = H$.

That $\min_\preceq(H^\succ)$ and $\min_\preceq(H^{\textrm{u}})$ are finite
follows directly from that $\hypospace$ is well partially ordered.
For every hypothesis $h' \in H^\prec$ there is a $h \in H$ such that
$h' \preceq h$ (by construction) and such that $h$ is minimal in $H$.
Hence, the maximal elements in $H^\prec$ are exactly the minimal
elements in the set of parents of the hypotheses in $H$, and
thus this set is also finite by the well partial orderedness of
$\hypospace$. Since all three sets are finite, so is $P$.

If $h$ exhibits $\propanc(h')$ for some $h' \in H^\prec$, then
$h \preceq h' \prec h''$ for some $h'' \in H$. Since $h' \not\in H$,
by convexity, $h$ cannot be in $H$ either. Thus, all $h \in H$
exhibit $\neg\propanc(h')$ for all $h' \in H^\prec$.

If $h$ exhibits $\propdesc(h')$ for some $h' \in H^\succ$, then
$h'' \prec h' \preceq h$ for some $h'' \in H$. Analogously to the
previous case, because $h' \not\in H$ and $H$ is convex, $h$ cannot
be in $H$. Thus, all $h \in H$ exhibit $\neg\propdesc(h')$ for all
$h' \in H^\succ$.

Finally, if $h$ exhibits $\propdesc(h')$ for some $h' \in H^{\textrm{u}}$,
then $h' \preceq h$. $h$ cannot belong to $H$ because if it did, $h'$
would be related to some element in $H$, contradicting the construction
of $H^{\textrm{u}}$. Thus, all $h \in H$ exhibit $\neg\propdesc(h')$ for all
$h' \in H^{\textrm{u}}$.

In summary, each hypothesis $h \in H$ exhibits all properties in $P$.
Thus, $H \subseteq \hypos(P)$.

Now, let $h'$ be a hypothesis not in $H$. We know that $h'$ belongs
to at least one of $H^{\prec}$, $H^{\succ}$ or, $H^{\textrm{u}}$.
If $h' \in H^{\prec}$ then it is either maximal in $H^{\prec}$ or the
ancestor of a hypothesis that is maximal in $H^{\prec}$; in either
case, it exhibits $\propanc(h'')$ for some $h'' \in H^{\prec}$.
Likewise, if $h' \in H^{\succ}$ then it is either minimal in $H^{\succ}$
or the descendant of a hypothesis that is minimal in $H^{\succ}$, so
it exhibits $\propdesc(h'')$ for some $h'' \in H^{\succ}$. Finally,
$h' \in H^{\textrm{u}}$ then it is either minimal in $H^{\textrm{u}}$ or
the descendant of a hypothesis that is minimal in $H^{\textrm{u}}$, so
it exhibits $\propdesc(h'')$ for some $h'' \in H^{\textrm{u}}$. In all
three cases, $h'$ exhibits a property whose negation is in $P$, and
therefore $h' \not\in \hypos(P)$. 
Hence $\hypos(P) \subseteq H$.

So far, we have shown that if $H$ is convex, then it can be represented
by a finite conjunction of properties in $\propspace$.
To show the converse (only if), let $H$ be a non-convex set. This
means there are three hypotheses, $h_a$, $h_b$ and $h_c$, such that
$h_a \preceq h_c \preceq h_b$, $h_a, h_b \in H$ and $h_c \not\in H$.
(Since the three hypotheses are necessarily distinct, we have in fact
$h_a \prec h_c \prec h_b$.)

Suppose there is a property set $P$ such that $\hypos(P) = H$: $P$
must exclude $h_c$, that is, there must be at least one property
$p \in P$ that $h_c$ does not exhibit. There are only four ways to
construct such a property:

\noindent%
(1) $p = \propanc(h)$ for some strict ancestor $h \prec h_c$. But this
property also excludes $h_b$ from $\hypos(P)$, since $h_c \preceq h_b$.

\noindent%
(2) $p = \propdesc(h)$ for some strict descendant $h_c \prec h$. This
property excludes $h_a$, since $h_a \preceq h_c$.

\noindent%
(3) $p = \neg\propanc(h)$ for some descendant $h_c \preceq h$. (Note
that here, $h$ may be equal to $h_c$.) Again, this property excludes
$h_a$, since $h_a \preceq h_c$.

\noindent%
(4) $p = \neg\propdesc(h)$ for some ancestor $h \preceq h_c$ (which
may also equal $h_c$). This property excludes $h_b$, since
$h_c \preceq h_b$.

\noindent%
Thus, it is not possible to exclude $h_c$ from $\hypos(P)$ without
also excluding either $h_a$ or $h_b$. Therefore, since $H$ includes
both $h_a$ and $h_b$ but not $h_c$, $\hypos(P)$ cannot equal $H$.
\end{proof}

\subsection{Diagnostic Questions and Their Representations}
\label{sub::relevantsetsofproperties}

Next, we describe three different ``diagnostic questions''. Each
question is a specific test that provides a piece of information
about the diagnosis problem at hand. The strategies we present in
Section~\ref{sec::strategies} to explore the hypothesis space in search of the
minimal diagnosis use these questions as their main primitives
for interacting with the problem.

We show how each of the questions is formulated as sets of hypotheses
to test, and how those hypothesis sets can be represented by
(conjunctive) sets of properties. In most cases, the mapping from a
question to a test and from a test to its representation is
straightforward, but for some, there are alternative representations.
Which is the best representation depends in part on the strategy for
exploring the hypothesis space: For conflict-directed strategies
(introduced in Subsection~\ref{sec::conflicts}), the representation
should produce conflicts that are as general as possible. In addition,
for the preferred-first strategy (Subsection~\ref{sub::pfs}),
those conflicts should generate as few successors as possible.
Finally, the property set should facilitate the task of the test solver.

\begin{question}\label{ques::candidate}
  Is a given hypothesis $h$ a diagnosis candidate? ($\candidate{h}$)
  \begin{itemize}
  \item {\bf Test hypothesis set}: $H = \{h\}$.
  \item {\bf Representation by properties}: 
    $\{\propdesc(h)\} \cup \{\neg\propdesc(h') \mid h' \in \children(h)\}$.
  \item {\bf Test result}: yes or no. The test solver returns $h$
    if successful, and $\bot$ otherwise.
  \end{itemize}
\end{question}

Note that this question could also be represented by the property set
$\{\propdesc(h),\propanc(h)\}$ (since $h$ is the only hypothesis that
is both an ancestor and a descendant of $h$).
However, the representation given above is the better one for the
conflict-directed preferred-first strategy, and the basis of the one
that we use.
For particular hypothesis spaces, there can also be other, simpler but
equivalent ways of representing them by properties. We discuss some
alternatives in conjunction with the SAT-based implementation of a test
solver for discrete event system diagnosis in Subsection~\ref{sub::dessat}.

\begin{question}\label{ques::minimal}
  Is a given candidate $\delta$ minimal? ($\minimal{\delta}$)
  \begin{itemize}
  \item {\bf Test hypothesis set}: 
    $H = \{h \in \hypospace \mid h \prec \delta\}$; 
  \item {\bf Representation by properties}: 
    $\{\propanc(\delta), \neg\propdesc(\delta)\}$.  
  \item {\bf Test result}: Testing $H$ above amounts to asking,
   ``is there a candidate preferred to $\delta$?''. Thus, the
   answer to the original question ($\delta$ is minimal) is yes 
   if the outcome of the test is $\bot$. If $\delta$ is not
   minimal, the test solver returns a strictly preferred candidate.
  \end{itemize}
\end{question}

\begin{question}\label{ques::coverage}
  Given a finite and explicitly enumerated set of hypotheses $S$,
  does $S$ cover the diagnosis? ($\coverage{S}$)
  \begin{itemize}
  \item {\bf Test hypothesis set}: 
    $H = \{h \in \hypospace \mid \forall h' \in S : h' \not\preceq h\}$; 
  \item {\bf Representation by properties}: 
    $\{\neg\propdesc(h') \in \propspace \mid h' \in S\}$.  
  \item {\bf Test result}: As in Question \ref{ques::minimal},
    testing $H$ asks the reverse of the question; thus, the answer
    is yes ($S$ does cover the diagnosis) if the test solver returns
    $\bot$, and if $S$ does not cover the diagnosis, it returns a
    counter-example, in the form of a candidate not covered by $S$.
  \end{itemize}
\end{question}

{
It is possible to characterise the minimal diagnosis 
in terms of diagnosis questions.

\begin{theo}\label{theo::mindiagasthreequestions}
  A subset of hypothesis $S$ is the minimal diagnosis
  if and only if it satisfies the following three conditions:
  \begin{itemize}
  \item
    $\forall h \in S.\ \candidate{h}$;
  \item
    $\forall h \in S.\ \minimal{h}$;
  \item
    $\coverage{S}$.
  \end{itemize}
\end{theo}

\begin{proof}
  That the minimal diagnosis satisfies these three conditions
  is a direct consequence of its definition 
  (Definition~\ref{def:minimal-diagnosis}).

  Assume now that $S$ satisfies the conditions of the theorem.  
  We show that $S = \min_{\preceq}(\Delta)$ 
  which,
  by Theorem~\ref{theo::minimaldiagnosis}, concludes the proof.
  Assume that $S \neq \min_{\preceq}(\Delta)$; 
  this means that either $S \setminus \min_{\preceq}(\Delta) \neq \emptyset$ 
  or $\min_{\preceq}(\Delta) \setminus S \neq \emptyset$.  

  i) Let $h$ be a hypothesis of $S \setminus \min_{\preceq}(\Delta)$; 
  $h$ is a candidate (by definition of $S$) but a non-minimal one.  
  Consequently, there exists a minimal candidate 
  $\delta \in \min_{\preceq}(\Delta)$ such that $\delta \prec h$.  
  This contradicts the condition $\minimal{h}$. 

  ii) Let $\delta$ be a minimal candidate 
  of $\min_{\preceq}(\Delta) \setminus S$.   
  Since $S$ covers the diagnosis, 
  it must contain a hypothesis $h \preceq \delta$; 
  furthermore, since $\delta \not\in S$, 
  $h \prec \delta$.  
  Because $\delta$ is a minimal candidate, $h$ is not a candidate.  
  This contradicts the first condition 
  that all hypotheses in $S$ should be candidates.  
\end{proof}

Some of our diagnosis procedures will not rely on the diagnosis question $\minimal{h}$.
For these procedures, we will rely on the following theorem instead.

\begin{theo}\label{theo::plscharacterisationofmindiag}
  A subset of hypotheses $S$ is the minimal diagnosis 
  if and only if it satisfies the following three conditions: 
 \begin{itemize}
 \item 
   $\forall h \in S,\ \candidate{h}$; 
 \item 
   $\forall h,h' \in S,\ h' \not\prec h$;
 \item 
   $\coverage{S}$.
 \end{itemize}
\end{theo}

\begin{proof}
  The proof is essentially the same as that of Theorem~\ref{theo::mindiagasthreequestions}.
  The difference lies in the part i).
  We reuse the same notation, i.e., $h \in S \setminus \min_{\preceq}(\Delta)$
  and $\delta \in \min_{\preceq}(\Delta)$ is such that $\delta \prec h$.
  From the third condition, we know that there is $h' \in S$ such that $h' \preceq \delta$
  (actually, the stronger relation $h' \prec \delta$ holds since $\delta$ is not an element of $S$).
  Therefore the two elements $h$ and $h'$ from $S$ satisfy $h' \prec h$,
  which contradicts the second condition of Theorem~\ref{theo::plscharacterisationofmindiag}.
\end{proof}
}

\section{Diagnostic Properties in Different Settings}
\label{sec::prop}
In this section, 
we illustrate how the abstract definitions in the previous section
are instantiated in two different modeling frameworks: static and
dynamic (discrete event) systems. For the latter, we also show the
instantiation of different hypothesis spaces.

\subsection{Static Systems}

Static systems are systems whose state does not normally change
over time (except for becoming faulty). A typical example of a
static system is a Boolean circuit. 
Static systems diagnosis consists in identifying the set of faults
the system exhibits at a given point in time; there is
no notion of multiple occurrences of the same fault, nor of
temporal order between fault occurrences.
Hence, the diagnosis is normally defined over the set hypothesis space
(power set of the set $F$ of faults), 
with the preference order defined as the subset relation.

Static systems are typically modeled 
by a finite set of variables, each with their own domain of values.
The set of possible system behaviours, which is a subset of all
assignments of values to the variables, is defined by a set $\MOD$ of
constraints over the variables. (These can be expressed in
propositional or first-order logic, or some other constraint
formalism.)
The observation is also defined by a set $\OBS$ of constraints 
on the values of certain variables of the model 
(for instance $\texttt{voltage} = \textrm{low}$).  
Each possible fault $f \in F$ is modeled  by a Boolean variable
$v_f \in V_F$: 
this variable takes the value \textit{true} 
iff the fault is present.
The hypothesis associated with a behaviour 
is then the subset of faults $f$ 
whose corresponding variable $v_f$ is \emph{true}: 
$\hypo(\beh) = \left\{ f \in F \mid \beh \rightarrow v_f\right\}$.

A hypothesis $h$ can be represented 
by a propositional formula $\Phi_h = 
\bigwedge_{f \in h} v_f \land \bigwedge_{f \in F \setminus h} \neg v_f$.  
A hypothesis $h \subseteq F$ is a candidate 
if it is logically consistent with the model and observation,
i.e., if 
\begin{displaymath}
  \MOD, \OBS, \Phi_h
  \not\models \bot.  
\end{displaymath}
Performing a test is therefore equivalent to solving a constraint
satisfaction problem. (In case the model is represented by a
propositional logic formula, that means a propositional satisfiability
problem).

The property $p_H$ corresponding to a hypothesis set
$H \subseteq \hypospace$, i.e., such that $\hypos(p_H) = H$, 
is the logical disjunction of the formulas of the hypotheses in $H$: 
$\Phi_{p_H} = \bigvee_{h \in H} \Phi_h$. Of course, $\Phi_{p_H}$ can
also be represented by any formula that is equivalent to this.
It is easy to show that: 
\begin{itemize}
\item $\Phi_{\propdesc(h)} \equiv \bigwedge_{f \in h} v_f$.
  That is, the descendants of hypothesis $h$ (which are those
  hypotheses that $h$ is preferred or equal to) are exactly
  those that include all faults that are present in $h$, and
  possibly other faults as well.
\item $\Phi_{\propanc(h)} \equiv \bigwedge_{f \in F \setminus h} \neg v_f$.
  That is, the ancestors of hypothesis $h$ (which are those
  hypotheses that are preferred or equal to $h$) are exactly
  those that do not include any fault not present in $h$,
  and possibly exclude some of the faults that $h$ has.
\end{itemize}

\subsection{Discrete Event Systems}
\label{sub::prop::dynamic}

Event-driven dynamic systems are characterised by transitions (be they
discrete, timed or continuous) taking place over time.
To simplify the discussion of this example, we will consider discrete
untimed transitions, i.e., the classical discrete event system (DES)
framework \cite{cassandras-lafortune::99}.
However, the formulation below, and the diagnosis algorithms we
present in the next section, generalise to other types of dynamic
systems.

Let $\Sigma$ be the set of events that can take place in the system.
A behaviour $\beh \in \Sigma^\star$ of the system is a (finite) sequence
of events. Thus, the system model is a language
$\MOD \subseteq \Sigma^\star$.
It is common to distinguish in $\Sigma$ a subset of observable events
($\Sigma_o$), and to define the observable consequence of a behaviour
as the projection $\Pi_{\Sigma_o}(\beh)$
of the event sequence $\beh$ on $\Sigma_o$
\cite{sampath-etal::tac::95}. Then, an observation, expressed as a
predicate on behaviours, has the form
$\OBS(\beh) \equiv (\Pi_{\Sigma_o}(\beh) = w)$, for some fixed
$w \in \Sigma_o^\star$. More general forms of observation, such as
partially ordered or ambiguous occurrences of observable events,
can be specified similarly.
Whichever form it takes, we can say that the observation is another
language $\mathcal{L}_{\OBS} \subseteq \Sigma^\star$ 
such that a behaviour $\beh$
is consistent with it iff $\beh \in \mathcal{L}_{\OBS}$.
The faults are modeled by a subset $F \subseteq \Sigma$ of
(unobservable) events.
The set of behaviours that correspond to a hypothesis $h$ is also
a language: $\mathcal{L}_h = \{ \beh \in \Sigma^\star \mid 
\hypo(\beh) = h \}$. The precise definition of $\mathcal{L}_h$
depends on the type of hypothesis space.
 
In most cases, these languages are all regular, and hence
representable by finite state machines. However, such a
representation is normally too large to be feasible for
computational purposes, so in practice an exponentially compact
factored representation, such as a network of partially
synchronised automata \cite{pencole-cordier::aij::05}, Petri nets \cite{benveniste:etal:03}, or description in a modelling
formalism like PDDL \cite{haslum-grastien::spark::11}, is used
instead.
As we describe the hypotheses and properties for different spaces
in the following, we will simply give them as (regular) languages.

For the set hypothesis space (SHS), 
a hypothesis $h \subseteq F$ corresponds to the language
$\mathcal{L}_h = 
\bigcap_{f \in h} (\Sigma^\star \{f\} \Sigma^\star)
\ \cap \ 
\bigcap_{f \in F \setminus h} (\Sigma \setminus \{f\})^\star$. 
For the multiset hypothesis space (MHS), 
the language $\mathcal{L}_h$ of hypothesis $h$ is the intersection
$\bigcap_{f \in F} \mathcal{L}_f^{=h(f)}$, where for each $f$,
$\mathcal{L}_f^{=h(f)}$ contains all event sequences that have
exactly $h(f)$ occurrences of $f$.
For instance 
$\mathcal{L}_f^{=2} = \left(\Sigma \setminus \{f\}\right)^\star \{f\}
\left(\Sigma \setminus \{f\}\right)^\star \{f\}
\left(\Sigma \setminus \{f\}\right)^\star$.  
For the sequence hypothesis space (SqHS), 
$\mathcal{L}_h$ is the language of words whose projection over
$F$ is $h$: if $h = [f_1,\dots,f_k]$, then $\mathcal{L}_h = 
\left(\Sigma \setminus F\right)^\star \{f_1\}
\left(\Sigma \setminus F\right)^\star \dots 
\left(\Sigma \setminus F\right)^\star \{f_k\}
\left(\Sigma \setminus F\right)^\star$.  

A hypothesis $h$ is a candidate 
if the intersection $\MOD \cap \mathcal{L}_{\OBS} \cap \mathcal{L}_h$ 
is non-empty.
Essentially, any $\beh$ that belongs to this intersection 
is a possible behaviour of the system.
Thus, a test can be seen as a discrete-state reachability problem.
Given compact representations of the languages involved, tests can
be carried out using, for example, model checking
\cite{clarke-etal::00} or
AI planning \cite{ghallab-etal::00} tools.

The property $p_H$ is also a language, specifically
$\mathcal{L}_{p_H} = \bigcup_{h \in H} \mathcal{L}_h$.  
Immediate from the definition, the language of a set of properties
$P$ is the intersection of the properties' languages: 
$\mathcal{L}_P = \bigcap_{p \in P} \mathcal{L}_p$.  
Likewise, the language of the negation of a property 
is the complement of its language, 
i.e., $\mathcal{L}_{\neg p} = \Sigma^\star \setminus \mathcal{L}_p$.
Using these, the languages of properties $\propdesc(h)$
and $\propanc(h)$ can be built up according to their definitions.

However, just as in the case of static systems, we can also find
simpler, and more intuitive, equivalent expressions for
$\mathcal{L}_{\propdesc(h)}$ and $\mathcal{L}_{\propanc(h)}$.
For the set hypothesis space, these are:
\begin{itemize}
\item 
  $\mathcal{L}_{\propdesc(h)} = 
  \bigcap_{f \in h} (\Sigma^\star \{f\} \Sigma^\star)$.
  In other words, descendants of $h$ are those event sequences
  that contain at least one occurrence of each fault $f \in h$.
\item 
  $\mathcal{L}_{\propanc(h)} = 
  \bigcap_{f \in F \setminus h} (\Sigma \setminus \{f\})^\star$.
  The ancestors of $h$ are those event sequences that do not
  contain any occurrence of any fault event not in $h$.
\end{itemize}
For the multiset hypothesis space, the languages of these
properties can be written as follows:
\begin{itemize}
\item 
  $\mathcal{L}_{\propdesc(h)} = 
  \bigcap_{f \in F} \mathcal{L}_f^{\ge h(f)}$, 
\item 
  $\mathcal{L}_{\propanc(h)} = 
  \bigcap_{f \in F} \mathcal{L}_f^{\le h(f)}$, 
\end{itemize}
where $\mathcal{L}_e^{\ge x}$ is the language of event sequences
in which $e$ occurs at least $x$ times 
and $\mathcal{L}_e^{\le x}$ the language of sequences
where $e$ occurs at most $x$ times.
The former can be written as $\mathcal{L}_e^{\ge x} =
\Sigma^\star \mathcal{L}_e^{=x} \Sigma^\star$,
and the latter as $\bigcup_{i=0,\ldots,x} \mathcal{L}_e^{=i}$.

For the sequence hypothesis space, the properties can be written
as follows. Let $h = [f_1,\dots,f_k]$: 
\begin{itemize}
\item 
  $\mathcal{L}_{\propdesc(h)} = 
  \Sigma^\star \{f_1\}
  \Sigma^\star \dots 
  \Sigma^\star \{f_k\}
  \Sigma^\star$.
  In other words, the descendants of $h$ are all event sequences
  in which the sequence $h = [f_1,\dots,f_k]$ is ``embedded''.
\item 
  $\mathcal{L}_{\propanc(h)} = 
  \left(\Sigma \setminus F\right)^\star \{f_1\}^{0/1} 
  \left(\Sigma \setminus F\right)^\star \dots 
  \left(\Sigma \setminus F\right)^\star \{f_k\}^{0/1}
  \left(\Sigma \setminus F\right)^\star$.
  That is, the ancestors of $h$ are all event sequences that
  contain some substring of $h$ as an embedded subsequence,
  and that do not contain any fault event interspersed between
  the occurrences of these fault events.
\end{itemize}


\section{Diagnosis Strategies}
\label{sec::strategies}

We have cast the diagnosis problem as a search for the minimal
candidates in the space of hypotheses, and we have shown how this
search can query the problem using symbolic
tests. To
instantiate the framework into a concrete diagnosis algorithm, we must
also specify a strategy for the exploration of the hypothesis space,
and an implementation of the test solver that is appropriate for the
class of system models and the hypothesis space. We describe
implementations of test solvers in Section \ref{sec::imple}.

In this section, we outline two broad types of exploration
strategies: The first, which we call ``preferred-last'',
maintains a set of candidates, which is iteratively extended until
it covers the diagnosis. The second, which we call
``preferred-first'', searches in a top-down fashion, testing
at each step the most preferred hypothesis that has not yet been
rejected. 
In each case, we first present the basic strategy,
followed by refined versions. In particular, we show how
the preferred-first strategy can be enhanced through the use of
\textit{conflicts}, in a manner analogous to their use in 
\texttt{diagnose} \cite{reiter::aij::87}.  

\subsection{The Preferred-Last Strategy}

The preferred-last strategy (PLS) begins with an empty set $S$ of
candidates, and repeatedly tests whether this set covers the
diagnosis. This test is an instance of Question \ref{ques::coverage},
described in Subsection~\ref{sub::relevantsetsofproperties}.  
If the answer to the question is
negative, it leads to the discovery of a new candidate which is
added to $S$. When $S$ covers the diagnosis we know that it is a
superset of the minimal diagnosis, because it contains only candidates.
The minimal diagnosis is then extracted from $S$ by removing non-minimal
elements, as 
required by Theorem~\ref{theo::plscharacterisationofmindiag}. 
The strategy is summarised in Algorithm~\ref{algo::pls}. 

\begin{algorithm}[ht]
  \begin{algorithmic}[1]
    \STATE {\bf Input}: Model $\MOD$, observation $o$, 
    hypothesis space $\hypospace$
    \STATE $S := \emptyset$
    \WHILE{$\neg \coverage{S}$}
      \STATE Let $\delta$ be the candidate found by the coverage test.  
      \STATE $S := S \cup \{\delta\}$
    \ENDWHILE
    \RETURN $\min_{\preceq} (S)$
  \end{algorithmic}
  \caption{The preferred-last strategy (PLS)}
  \label{algo::pls}
\end{algorithm}

\begin{theo}
PLS returns the minimal diagnosis. Furthermore, if the hypothesis space is well partially ordered,
then PLS terminates.
\end{theo}

\begin{proof}
  Assume PLS terminates. We first show that the three conditions 
  of Theorem~\ref{theo::plscharacterisationofmindiag} are satisfied 
  by the returned set $R = \min_{\preceq}(S)$.  
  Observe that both $S$ and $R \subseteq S$ are finite 
  since $S$ is enumerated.  
  1) All hypotheses in $R$ are candidates.  
  2) Since $R$ is minimised, it contains no pair of hypotheses 
  that are comparable.  
  3) Let $\delta \in \Delta$ be a candidate.  
  Since $S$ covers $\Delta$, 
  there exists $h_1 \in S$ such that $h_1 \preceq \delta$.  
  If $h_1 \in R$, then $\delta$ is covered, 
  but we need to consider the general case where $h_1 \not\in R$.
  Because $h_1$ is in the set of non-minimal elements of $S$ and $S$ is finite,
  there is another hypothesis $h_2 \in S$ such that $h_2 \preceq h_1$ holds.
  This hypothesis $h_2$ could also not belong to $R$,
  in which case this hypothesis is also covered by another hypothesis $h_3$.
  This gives us a sequence of hypotheses $h_1 \succ h_2 \succ \dots$
  that all belong to $S$.
  Since $S$ is finite, there is a minimal hypothesis $h_k$ for this sequence,
  and this hypothesis belong to $\min_{\preceq}\ S$.  
  Thus $R$ covers the diagnosis.  

{
  Now, suppose that PLS does not terminate: This means PLS generates an
  infinite sequence of candidates, $\delta_1, \delta_2, \ldots$
  Because $\delta_j$ is generated from a test of coverage of $\{\delta_1,\dots,\delta_{j-1}\}$,
  we know that $\delta_i \not\preceq \delta_j$ for all $i < j$.
  Furthermore, since the preference order is well-founded,
  we know that any strictly descending subchain of this sequence is finite.
  Therefore, for any index $i$, there exists at least one index $k \ge i$
  such that $\delta_k \preceq \delta_i$ and $\delta_k$ is minimal in the sequence.
  We write $m(i)$ the smallest such index $k$.
  We note that for any index $j > m(i)$,
  $\delta_{m(i)}$ and $\delta_j$ are incomparable
  (as $\delta_{m(i)}$ is minimal in $S$ and $\delta_j$ is after $\delta_{min(i)}$ in the sequence).
  We also note $m(i+1) > i$ for any index $i$.
  Therefore, the set
  \begin{displaymath}
    S' = \{ \delta_{m(i)}, \delta_{m(m(i)+1)}, \delta_{m(m(m(i)+1)+1)}, \dots \}
  \end{displaymath}
  contains infinitely many mutually-incomparable candidates (hence, all minimal in $S'$),
  which contradicts the well partial orderness of $\preceq$.
}
\end{proof}

Although the PLS algorithm is guaranteed to eventually terminate, for
infinite hypothesis spaces there is no worst-case bound on the number of
iterations required before a covering set has been found (for finite
hypothesis spaces it is of course bounded by the size of the space).
Consider, for instance, the Sequence Hypothesis Space with only one fault $f$
and write $h_i = f^i$
(i.e., $h_i$ indicates that $f$ occurred precisely $i$ times);
assume that the diagnosis is $\Delta = \{h_0,h_1,h_2,\dots\} = \hypospace$
(any number of occurrences of $f$ might have happened);
then for any $i$, PLS could generate this sequence of candidates:
$h_i, h_{i-1}, h_{i-2}, \dots, h_0$.
All sequences will eventually end with $h_0$,
but there is no a-priori bound on their size until (in this instance) the first candidate is found.

PLS computes some candidates and then tries to improve them.
Sometimes, however, instead of improving known candidates,
it will go sideways and compute other irrelevant candidates.
The following example illustrates this problem of slow convergence.

\begin{exam}\label{ex::pls-slow-convergence}
  Consider a set hypothesis space over a large set of faults $F$, 
  and a diagnosis problem in which $\Delta = \hypospace$, i.e., all
  hypotheses are candidates 
  (this would be the situation for example 
  in a weak-fault model with nominal observations).  
  The minimal diagnosis is then the singleton $\Delta_{\preceq} = \{h_0\}$.

  All candidates that involve $\lfloor \frac{|F|}{2}\rfloor$ faults 
  are mutually incomparable, which means the coverage test can
  iteratively generate all of them, leading to an exponential-time
  computation.
\end{exam}

In order to speed up convergence of PLS, 
we add an extra step which ``refines'' each new candidate found into a minimal
one. The intuition is that if minimal candidates are generated early,
we can avoid exploring ``redundant'' options. 
For instance, in Example \ref{ex::pls-slow-convergence} above, 
the number of iterations will be at most $|F| + 1$.  

The refinement of a candidate $\delta$ is performed by testing
whether $\delta$ is minimal, i.e., asking Question
\ref{ques::minimal}. If  $\delta$ is not minimal, the test returns
a preferred candidate; this is repeated until the current candidate
is minimal.
The revised algorithm, called PLS+r, is shown in Algorithm~\ref{algo::plsr}.
Note that, in this algorithm, all elements inserted in $S$ are guaranteed
to be minimal. Thus, there is no need to remove non-minimal elements at
the end.

\begin{algorithm}[ht]
  \begin{algorithmic}[1]
    \STATE {\bf Input}: Model $\MOD$, observation $o$, 
    hypothesis space $\hypospace$
    \STATE $S := \emptyset$
    \WHILE{$\neg \coverage{S}$}
      \STATE Let $\delta$ be the candidate found by the coverage test.
      \WHILE{$\neg \minimal{\delta}$}
        \STATE Replace $\delta$ with the candidate 
        found by the minimality test.  
      \ENDWHILE
      \STATE $S := S \cup \{\delta\}$
    \ENDWHILE
    \RETURN $S$
  \end{algorithmic}
  \caption{The preferred-last strategy with refinement (PLS+r)}
  \label{algo::plsr}
\end{algorithm} 

\begin{theo}
PLS+r returns the minimal diagnosis. Furthermore, if the hypothesis
space is well partially ordered, then PLS+r terminates.
\end{theo}

\begin{proof}
  Any candidate added to $S$ by PLS that is not also added by PLS+r
  is non-minimal, and therefore removed from the final set by PLS.
  Thus, PLS+r returns the same diagnosis. The refinement step
  effectively changes only the order in which candidates are
  generated. Since PLS terminates regardless of the order in which the
  candidates are generated in, PLS+r also terminates under the same
  condition.
\end{proof}

\subsection{The Preferred-First Strategy}
\label{sub::pfs}

The preferred-first strategy is based on the following intuition:
Because faults are rare events,
it can be expected that minimal candidates have small depth.  
Therefore, a sensible approach to the hypothesis space exploration 
is to start by testing the most preferred hypotheses; 
if those hypotheses are proven to be candidates,
then their descendants do not need to be explored,
since we are only interested in the minimal diagnosis. 
The basic version of the preferred-first strategy (PFS) 
is presented in Algorithm~\ref{algo::pfs}.   

\begin{algorithm}[ht]
  \begin{algorithmic}[1]
    \STATE {\bf Input}: Model $\MOD$, observation $o$, 
    hypothesis space $\hypospace$
    \STATE $\SR := \emptyset$
    \COMMENT{Will store the result}
    \STATE $\SO := \min_{\preceq}(\hypospace)$
    \COMMENT{i.e., $\{h_0\}$}
    \WHILE{$\SO \neq \emptyset$}
    \STATE $h := \textrm{pop}(\SO)$\label{line::pfs::pop}
    \IF{($\exists h' \in \SO \cup \SR: h' \preceq h$)}
    \label{line::pfs::subsumption}
    \STATE {\bf continue}
    \ENDIF\label{line::pfs::endsubsumption}
    \IF {$\candidate{h}$}\label{line::pfs::test}
    \STATE $\SR := \SR \cup \{h\}$\label{line::pfs::addingcandidate}
    \ELSE
    \STATE $\SO := \SO \cup \children(h)$\label{line::pfs::addingchildren}
    \ENDIF\label{line::pfs::endexpansion}
    \ENDWHILE
    \RETURN $\SR$
  \end{algorithmic}
  \caption{The preferred-first strategy (PFS).}
  \label{algo::pfs}
\end{algorithm}

Both $\SO$ and $\SR$ are enumerated sets of hypotheses and,
because any hypothesis has only a finite set of children, 
both sets are guaranteed to be finite.
The set $\SO$ contains all hypotheses that are ``promising'',
in the sense that their parents have been ruled out as
candidates but the hypotheses themselves have not yet been
tested.
Starting with the unique most preferred hypothesis $h_0$, 
the algorithm selects a current hypothesis 
$h$ to test, which it removes from $\SO$
and stores it in $\SR$ if it is a candidate; 
otherwise, it adds the children of $h$ to $\SO$.

PFS returns the correct diagnosis, but termination is only ensured if
the hypothesis space is finite.  To demonstrate these results, we
first prove the following lemma:

\begin{lemm}\label{lemm::coverage}
  Whenever the condition of the \textbf{while} loop in PFS is tested, 
  the diagnosis is covered by $\SO \cup \SR$, i.e.,
  $\forall \delta \in \Delta,\ \exists h \in \SO \cup \SR:\ h \preceq \delta$.
\end{lemm}

\begin{proof}
  We prove the lemma by induction.  

  Initially, 
  $\SO = \{h_0\}$ so the coverage property holds.  

  Assume that the coverage property is true for some $\SO \neq \emptyset$
  and some $\SR$. 
  We prove that the property still holds after a single execution of the
  loop body.  
  Let $h \in \SO$ be the hypothesis chosen at Line~\ref{line::pfs::pop}.  
  Consider a candidate $\delta$:
  by induction, we know that there exists $h' \in \SO \cup \SR$ 
  such that $h' \preceq \delta$.  
  If $h' \neq h$, then the condition still holds in the next iteration,
  since $h'$ remains in  $\SO \cup \SR$.
  On the other hand, if $h' = h$, then there are three cases: 
  
  \noindent i) If the condition on Line~\ref{line::pfs::subsumption} is true, 
  then there exists $h'' \in (\SO \setminus \{h\}) \cup \SR$ such that $h'' \preceq h \preceq \delta$.  
  Since $h''$ remains in $\SO \cup \SR$ at the start of the next iteration,
  candidate $\delta$ is covered.

  \noindent ii) If the condition on Line~\ref{line::pfs::test} is true, 
  $h$ is simply moved from $\SO$ to $\SR$, 
  so $\SO \cup \SR$ remains unchanged
  and the coverage property holds by induction.

  \noindent iii) If neither of these two conditions is satisfied,
  $h$ will be removed from $\SO$ and its children added instead.
  In this case, since $h \preceq \delta$ but $h$ cannot be equal
  to $\delta$, we have $h \prec \delta$. Hence, there exists
  at least one hypothesis $h''$ such that $h \prec h'' \preceq \delta$,
  and any minimal such hypothesis is a child of $h$.
  Hence candidate $\delta$ is covered at the next iteration by at least
  one child $h''$ of $h$ that has been added to $\SO$.
\end{proof}

\begin{theo}\label{theo::pfs::correct}
PFS returns the minimal diagnosis. Furthermore, if the hypothesis
  space is finite, then PFS terminates.
\end{theo}

\begin{proof}
  Let $\SR$ be the result of the algorithm (assuming it terminates).  
  We prove that $\SR \subseteq \Delta_\preceq$,
  and then that $\Delta_\preceq \subseteq \SR$.  

  $\SR$ is initially empty and elements are added 
  (Line~\ref{line::pfs::addingcandidate})
  only when they are proved to be candidates: 
  hence $\SR \subseteq \Delta$.  
  Furthermore we know 
  from Lemma~\ref{lemm::coverage}
  that $\SR \cup \SO$ covers the diagnosis at all times.  
  Assume the non-minimal diagnosis candidate $h = \delta$ is added to $\SR$
  in some iteration.
  This means that $\delta$ is the hypothesis popped from $\SO$ in this
  iteration.
  Since $\delta$ is non-minimal, 
  there exists a preferred candidate $\delta' \preceq \delta$ 
  and this candidate is covered: $\exists h' \in \SR \cup \SO.\ h' \preceq \delta'$.  
  This, however, means that $h' \preceq \delta$, so $\delta$
  cannot have not have passed the test at Line~\ref{line::pfs::subsumption}. 
  Hence, $\SR$ contains only minimal candidates.  

  At the end of the algorithm, $\SO$ is empty, so $\SR$ alone covers
  the diagnosis. Hence, for any minimal candidate $\delta$, 
  there exists a hypothesis preferred to $\delta$ that appears in $\SR$.  
  But $\SR$ contains only minimal candidates, 
  and the only candidate $\delta'$ 
  that satisfies $\delta' \preceq \delta$ is $\delta$.  
  Therefore, all minimal candidates appear in $\SR$.  

To show termination, 
we prove that $\SO$ eventually becomes empty.  

  At each iteration, one hypothesis $h$ is removed from $\SO$;
  under certain conditions, the children of $h$ are also added
  to $\SO$. We show that when this happens, the hypothesis $h$
  that was just removed can never re-enter $\SO$ in any future
  iteration.

  A hypothesis $h$ can be added to $\SO$ only in the interaction in which
  one of its parents was removed from $\SO$. Thus, if no ancestor
  of $h$ is currently in $\SO$ then $h$ cannot be added to $\SO$ in
  any future iteration.

  Consider a hypothesis $h$, removed from $\SO$ in the current
  iteration, and suppose that the algorithm reaches
  Line~\ref{line::pfs::addingchildren}, so that children of $h$ are
  added to $\SO$. This means the condition on
  Line~\ref{line::pfs::subsumption} does not hold, which means
  there is no ancestor of $h$ in $\SO$ (or in $\SR$). Hence, $h$
  can never re-enter $\SO$.
\end{proof}

In general, there is no guarantee 
that PFS will terminate 
when the hypothesis space is infinite.

This is illustrated by the two examples below. In the first, lack of termination comes from \textit{useless} hypotheses which have no candidates among their descendants. As the second example shows even pruning those useless hypotheses is not sufficient to ensure termination.

\begin{exam}
Consider a SqHS with two faults $f_1$ and $f_2$, 
and suppose that the diagnosis is $\Delta = \{[f_1]\}$.  
Then, PFS will never end.  
Table~\ref{tab::pfsneverends} shows a possible evolution of PFS.  
PFS is unaware of the fact 
that no descendant of $[f_2,f_2,\dots,f_2]$ is a candidate, 
and will therefore explore this branch for ever.  

\begin{table}[ht]
  \begin{center}
  \begin{tabular}{| c | c | c |}
    \hline
    $\SO$ & $\SR$ & next element popped\\
    \hline
    $\{[]\}$ & $\{\}$ & $[]$\\
    $\{[f_1],[f_2]\}$ & $\{\}$ & $[f_1]$\\
    $\{[f_2]\}$ & $\{[f_1]\}$ & $[f_2]$\\
    $\{[f_1,f_2],[f_2,f_1],[f_2,f_2]\}$ & $\{[f_1]\}$ & $[f_1,f_2]$\\
    $\{[f_2,f_1],[f_2,f_2]\}$ & $\{[f_1]\}$ & $[f_2,f_1]$\\
    $\{[f_2,f_2]\}$ & $\{[f_1]\}$ & $[f_2,f_2]$\\
    $\{[f_1,f_2,f_2], [f_2,f_1,f_2], [f_2,f_2,f_1], 
        [f_2,f_2,f_2]\}$ & $\{[f_1]\}$ & $[f_1,f_2,f_2]$\\
    \dots & \dots & \dots \\
    \hline
  \end{tabular}
  \end{center}
  \caption{Possible evolution of PFS}
  \label{tab::pfsneverends}
\end{table}
\end{exam}

\begin{exam}
  Consider again a SqHS with two faults $f_1$ and $f_2$, 
  and consider that the diagnosis 
  is $\Delta = \{[f_1],[f_1,f_2],[f_1,f_2,f_2],\dots\}$, 
  i.e., any hypothesis that starts with $f_1$, 
  followed by any number of $f_2$.  
  Then, all hypotheses of the form
  $[f_2,\dots,f_2]$ have a child that is a candidate (the
  hypothesis with $f_1$ added to the beginning of the sequence),
  and hence none of them are useless. This makes it possible
  for PFS to explore an infinite path in the hypothesis
  space without 
  encountering any candidate, thus never
  terminating.
\end{exam}

However termination can also be guaranteed 
by pruning a different type of hypotheses.  
We call \textit{undesirable} those hypotheses that are not ancestors 
of any minimal candidates 
(formally, $\descendants(h) \cap \Delta_\preceq = \emptyset$).  
Again, assuming that all hypotheses have finite depth
then pruning undesirable hypotheses guarantees termination.  

In fact, we use an even stronger pruning condition, which discards
all undesirable hypotheses as well as some hypotheses that do not
satisfy the undesirability condition but are redundant because the
candidates they can lead to are covered by some other hypothesis.
We call these hypotheses \textit{non-essential}.
Pruning non-essential hypotheses works better than pruning only
the undesirable hypotheses for two reasons: First, because
the undesirability condition cannot be directly tested during
search, since the minimal diagnosis $\Delta_\preceq$ is not known;
the essentiality property, on the other hand, is straightforward to
test. Second, pruning more hypotheses, as long as it does not
compromise completeness of the returned diagnosis, is of course
preferable since it leads to less search.
Note that the part of the proof of Theorem~\ref{theo::pfse::correct}
that establishes termination does not actually depend on 
pruning non-essential hypotheses; it continues to hold also if only
undesirable hypotheses are pruned.
 
A hypothesis $h$ is said to be non-essential,
with respect to $\SO$ and $\SR$,
if all candidates $\delta$ that are descendants of $h$ 
are also descendants of some other hypothesis in $\SO \cup \SR$.
The proof of Theorem~\ref{theo::pfs::correct} 
relies on the coverage property which states that
for every candidate $\delta$ some $h \preceq \delta$
(either $\delta$ or one of its ancestors)
appears in $\SO \cup \SR$ at the start of every iteration.
Therefore, if $(\SO \setminus \{h\}) \cup \SR$ covers the diagnosis,
then $h$ can be safely discarded from $\SO$ without losing the coverage
property. Because $\SO \cup \SR$ always covers the diagnosis 
(by Lemma~\ref{lemm::coverage}), $h$ is non-essential exactly when
$(\SO \setminus \{h\}) \cup \SR$ also covers the diagnosis.
Note that an undesirable hypothesis $h$ is always non-essential
w.r.t.\ $\SO$ and $\SR$ if $\SO \cup \SR$ covers the minimal
diagnosis. Therefore, any undesirable hypothesis will be pruned
by skipping non-essential hypotheses in PFS.
The non-essentiality test is shown in Algorithm~\ref{algo::pfse}.
It is added to PFS between Lines~\ref{line::pfs::endsubsumption}
and~\ref{line::pfs::test}. We call the resulting algorithm PFS+e.

\begin{algorithm}[ht]
  \begin{algorithmic}[1]
  \IF{$\coverage{\SO \cup \SR}$}\label{line::pfse::nonessential}
    \STATE {\bf continue} 
    \COMMENT{Note that $h$ is no longer in $\SO$ at this stage.}
  \ENDIF
  \end{algorithmic}
  \caption{Addition to PFS between Line~\ref{line::pfs::endsubsumption} 
    and Line~\ref{line::pfs::test}
    for ignoring non-essential hypotheses.}
  \label{algo::pfse}
\end{algorithm}

\begin{theo}\label{theo::pfse::correct}
  PFS+e returns the minimal diagnosis.  
  Furthermore, if all hypotheses of the hypothesis space have finite depth, 
  then PFS+e terminates.  
\end{theo}

\begin{proof}
  That PFS+e returns the minimal diagnosis can be shown
  simply by proving that the coverage property (Lemma~\ref{lemm::coverage}) 
  still holds.
  We now have a fourth case in the induction step of proof: If
  $h$ fails the non-essentiality test, then it is discarded without
  its children being added to $\SO$. However, this test 
  actually checks the very property 
  that we want to enforce, that $\SO \cup \SR$ covers the diagnosis,
  so when it triggers and returns to start of the {\bf while} loop,
  the coverage property also holds.

  Next, we consider termination.
  Let $\SO@0,\dots,\SO@i,\dots$ represent the content of the set $\SO$ 
  at the start of each iteration of the {\bf while} loop, when the
  loop condition is evaluated.
  We need to show that $\SO@i$ will eventually be empty.
  To do this, we make use of the following three facts 
  which we then prove:\\
  i) Let $A = \{ h \in \hypospace \mid 
  \exists \delta \in \Delta_\preceq.\ h \preceq \delta\}$:
  $A$ is finite.\\
  ii) $\SO@i \cap A = \SO@k \cap A \Rightarrow 
  \forall j \in \{i,\dots,k\}.\ \SO@j \cap A = \SO@i \cap A$. \\
  iii) $\SO@i \cap A = \SO@(i+1) \cap A \Rightarrow 
  \SO@(i+1) \subset \SO@i$.

  Assume the sequence $\SO@i$ goes on forever. 
  By (iii) and because $\SO$ is always finite,
  the intersection of $\SO$ and $A$ changes infinitely often.  
  Furthermore, by (i) there is only a finite number of intersections 
  of $\SO$ and $A$, which means that the same intersection 
  must eventually reappear. This contradicts (ii).

  It remains to prove claims (i) -- (iii).

  i) 
  First, note that that
  $A = \bigcup_{\delta \in \Delta_\preceq} \ancestors(\delta)$.
  Because $\Delta_\preceq$ is finite, $A$ is finite 
  iff the set of ancestors of every minimal candidate is finite.  
  Consider a minimal candidate $\delta$. 
  Since $\delta$ has finite depth (assumption of the theorem), 
  its ancestors all have depth $d$ or less.  
  We prove, by induction, that the set of hypotheses of depth
  $d$ or less is finite, for any $d$.
  This is true for $d = 1$, since only $h_0$ has depth $1$.  
  Assume it is true for $d-1$.  
  By definition of depth, every hypothesis $h$ of depth $d$ 
  is a child of some hypothesis $h'$ of depth $d-1$.  
  Since there is a finite number hypotheses $h'$ at depth $d-1$,
  by inductive assumption, and each of them has a finite number
  of children (because the hypothesis space is well partially
  ordered), there can only be a finite number of hypotheses with
  depth $d$. Thus, the number of hypotheses of depth $d$ or less
  is also finite.

  ii) 
  Assume $i < j < k$ such that 
  $\SO@i \cap A = \SO@k \cap A$ and 
  $\SO@i \cap A \neq \SO@j \cap A$.  
  Let $A' \subseteq A$ be the set of hypotheses $h$ 
  that are added at some point between iteration $i$ and iteration $k$,
  that is, $A' = \{h \in A \mid \exists \ell \in \{i,\dots,k-1\}.\ 
  h \not\in \SO@\ell \land h \in \SO@(\ell+1)\}$.
  Clearly $A'$ is not empty: Since $\SO@i \cap A \neq \SO@j \cap A$,
  some hypothesis has either been added between $i$ and $j$, or
  some hypothesis has been removed between $i$ and $j$, in which case
  it must added again before iteration $k$.
  Let $h$ be a hypothesis that is minimal in the set $A'$.
  Since $h$ is added to $\SO$ at some point 
  between iteration $i$ and iteration $k$, 
  a parent $h'$ of $h$ must be removed at the same iteration 
  (the only way to add an element to $\SO$ 
  is through Line~\ref{line::pfs::addingchildren}).
  However, if $h'$ is removed from $\SO$, 
  it must be added again to $\SO$ at some later point, as otherwise
  $\SO@k \cap A$ could not equal $\SO@i \cap A$. This means $h'$
  also belongs to $A'$, and since it is a parent of $h$, this
  contradicts the choice of $h$ as a minimal hypothesis in $A'$.

  iii) 
  Consider an iteration $i$ such that $\SO@i \cap A = \SO@(i+1) \cap A$.
  Because $\SO \cap A$ is unchanged, 
  the hypothesis $h$ chosen at iteration $i$ 
  does not belong to $A$.
  Any hypothesis not in $A$ is, by definition, undesirable, since $A$
  contains all ancestors of all minimal candidates. Thus, since
  $\SO@i \cup \SR@i$ covers the minimal diagnosis (by Lemma
  ~\ref{lemm::coverage}), so does $(\SO@i \cap A) \cup \SR@i$, and
  consequently so does $(\SO@i \setminus \{h\}) \cup \SR@i$.
  Thus, $h$ fails the essentiality test in PFS+e, so no children
  of $h$ are added to $\SO$ and we have $\SO@(i+1) = \SO@i \setminus \{h\}$.
\end{proof}

\subsection{Conflict-Based Strategy}
\label{sec::conflicts}

The conflict-based strategy is an improvement of PFS.  
The idea is to extract the core reason 
why hypothesis $h$ is not a candidate, 
in order to reduce the number of successors of $h$ 
that need to be inserted in the open list.  

We define a conflict as an implicit representation of a set of hypotheses 
that are not candidates.  
\begin{defi}
  A \emph{conflict} $C$ is an object that represents 
  a set $\hypos(C)$ of hypotheses 
  that does not intersect the diagnosis: 
  \begin{displaymath}
    \hypos(C) \cap \Delta = \emptyset.  
  \end{displaymath}
\end{defi}

We now assume that the test solver 
is not only able to decide whether the diagnosis 
intersects the specified set of hypotheses, 
but also to return a conflict in case the test fails.  
The following definition of a test result 
extends Definition~\ref{defi::test}.  

\begin{defi}
  The \emph{result} of a test $\langle \MOD, o, H\rangle$ 
  is either a hypothesis $h \in \Delta(\MOD,o,\hypospace) \cap H$ 
  or a conflict $C$ such that $H \subseteq \hypos(C)$.  
\end{defi}

In the worst case, i.e., if the test solver is not able 
to provide useful information, 
the conflict can be defined such that $\hypos(C) = H$.  

Two problems need to be solved at this stage: 
i) how do we compute conflicts, 
and ii) how do we use conflicts for diagnosis.  
We first concentrate on the second issue.  

\subsubsection{Using Conflicts for Diagnosis}

A conflict can be useful in two different ways.  

First, a conflict can be used to avoid certain tests.  
For instance, let $h$ be a hypothesis, 
the candidacy of which we want to test, 
and let $C$ be a conflict that was previously computed.  
If $h \in \hypos(C)$, 
then $h \not\in \Delta$ (by definition of a conflict).  
Therefore, inclusion in a conflict 
can serve as an early detection 
that a hypothesis is not a candidate.  

The second use of a conflict is to reduce the number of successors 
that need to be generated after a candidacy test failed.  
Again, let $h$ be a hypothesis 
and let $C$ be a conflict 
such that $h \in \hypos(C)$.  
Remember that the correctness of PFS 
relies on the fact 
that all diagnosis candidates are covered 
by a hypothesis from the open list 
or by a hypothesis from the already discovered minimal candidates.  
When $h$ is proved to be a non-candidate, 
we no longer need to get $h$ covered, 
but we need to cover the set $S$ of all strict descendants of $h$, 
which is the reason why 
Algorithm~\ref{algo::pfs} includes all the minimal elements of $S$ 
(the children of $h$) in the open list.  
Now however, not only do we know that $h$ is not a candidate, 
but the same also applies to all the hypotheses of $\hypos(C)$.  
Therefore, we may include in the open list 
the minimal elements of $S \setminus \hypos(C)$.  
This is illustrated with Algorithm~\ref{algo::conflict} 
where the conflict is used to compute the set of successors.  
We call PFS+ec (resp.\ PFS+c) the variant of PFS+e (resp.\ PFS) that uses conflicts.  

\begin{algorithm}
  \begin{algorithmic}
  \IF {$\candidate{h}$}
    \STATE $\SR := \SR \cup \{h\}$
  \ELSE
    \STATE Let $C$ be the conflict generated by the test.  
    \STATE $\SO := \SO \cup \min_{\preceq}
    (\descendants(h) \setminus \hypos(C))$\label{line::pfs+ec::update}
  \ENDIF
  \end{algorithmic}
  \caption{Replacement of the If statement 
  Lines~\ref{line::pfs::test}-\ref{line::pfs::endexpansion} of Algorithm~\ref{algo::pfs}.}
  \label{algo::conflict}
\end{algorithm}

\begin{theo}\label{theo::pfsecterminates}
  PFS+ec returns the minimal diagnosis.  
  Furthermore if all hypotheses of the hypothesis space 
  have finite depth, then PFS+ec terminates.  
\end{theo}

\begin{proof}
  The correct outcome of the algorithm 
  is again proved by updating the coverage property 
  (Lemma~\ref{lemm::coverage}).  
  Item iii) of the proof needs to be updated
  as follows.
  Candidate $\delta$ is covered by the hypothesis $h$ 
  that has been disproved ($\delta \in \descendants(h)$).  
  Because $\delta$ is a candidate and $C$ is a conflict, 
  $\delta \not\in \hypos(C)$.  
  Hence $\delta \in \descendants(h) \setminus \hypos(C)$.  
  Since the hypothesis space is well partially ordered, 
  $\min_\preceq(\descendants(h) \setminus \hypos(C))$ is
  not empty and therefore, when the hypotheses in this set are
  added to $\SO$ (line~\ref{line::pfs+ec::update}), at least one of them
  will cover $\delta$ at the next iteration of the algorithm.

  The proof for termination of PFS+e 
  also apply to PFS+ec.  
\end{proof}

We now illustrate how PFS+ec can accelerate the diagnosis.  
First it may remove a number of successors.  

\begin{exam}
  Consider a SqHS with three fault events $f_1$, $f_2$, and $f_3$.  
  PFS+ec first tests the empty sequence $h_0 = []$.  
  Assuming $h_0$ is not a candidate, 
  then PFS would generate three successors, $[f_1]$, $[f_2]$, and $[f_3]$.  
  Assume now the test solver finds the conflict $C$ 
  that specifies that \textit{either fault $f_1$ or fault $f_2$} occurred.  
  This conflict rejects all hypotheses that contain only $f_3$ faults.  
  It is not difficult to show that the minimal elements 
  of $\descendants(h_0) \setminus \hypos(C)$ 
  are $[f_1]$ and $[f_2]$.  
  In other words, the conflict allowed to discard the hypothesis $[f_3]$.  
\end{exam}

But conflicts can also allow us to consider hypotheses 
that are ``deeper'' than the natural successors, 
thus skipping intermediate steps.  

\begin{exam}
  Consider the same example as before, 
  but this time with the conflict $C$ 
  that excludes $h_0$ and all hypotheses with a single fault.  
  {Then the successors of $h_0$ become:}
  $[f_1,f_1]$, $[f_1,f_2]$, $[f_1,f_3]$, 
  $[f_2,f_1]$, $[f_2,f_2]$, $[f_2,f_3]$, 
  $[f_3,f_1]$, $[f_3,f_2]$, and $[f_3,f_3]$.
  {PFS does not need to test any $h_0$'s three children.}
\end{exam}

\subsubsection{Computing Conflicts and Successors}

So far, we have merely characterised the set of successors rejected
by a conflict, but have not explained how to compute
conflicts and successors in practice.  
A key issue addressed by our approach below is that the set $\left(\descendants(h) \setminus \hypos(C)\right)$ is infinite in general.

\vspace*{1ex}

We first discuss the \emph{computation of conflicts}. Whilst
our approach restricts the type of conflicts computed,
it makes it easy to test for inclusion in a conflict and compute successors.
Conflicts are represented symbolically, similarly to the tested hypotheses.
A conflict is a set of hypothesis properties which, 
as explained in Definition~\ref{defi::prop}, 
is an implicit representation of a set of hypotheses: 
\begin{displaymath}
  C \subseteq \propspace.  
\end{displaymath}
To see how conflicts are computed, remember that the test solver is
given a set $P$ of properties that represents exactly the set $H$ of
hypotheses
to be tested ($H = \hypos(P)$).  
The task of the test solver is essentially 
to find an ``explanation'' of the observation 
that satisfies all these properties.  
If no such explanation exists (and, consequently, the test fails), 
then the solver may be able to track all the properties $P'\subseteq P$ 
that it used to decide the failure.  
Clearly: 
\begin{itemize}
\item 
  no hypothesis that satisfies $P'$ is a candidate; 
  hence $P'$ is a conflict; 
\item 
  the set of hypotheses represented by $P'$ 
  is a superset of the set of hypotheses represented by $P$: 
  $P' \subseteq P \Rightarrow \hypos(P') \supseteq \hypos(P) = H$.  
\end{itemize}
Therefore, $P'$ is returned as the result of the diagnosis test.

\vspace*{1ex}

Given this definition of conflict, we now discuss the efficient \emph{computation of the successors} of a hypothesis rejected by a conflict.  
First,
observe that PFS searches using Question~\ref{ques::candidate}, which,
as stated in Subsection~\ref{sub::relevantsetsofproperties}, can be
formulated via
two properties of the form $\propdesc(\cdot)$
and $\propanc(\cdot)$, 
or alternatively 
via a $\propdesc(\cdot)$
property in conjunction with a set of $\neg\propdesc(\cdot)$
properties.
We choose the latter representation, as using more
properties will enable the generation of a 
more general conflict and increase efficiency.

Second, the property of the form $\propdesc(\cdot)$
can be ignored for the purpose of computing successors. This is because
the successors of $h$ (as defined in Algorithm~\ref{algo::conflict})
should contradict at least one property of the conflict but cannot
contradict a $p=\propdesc(h')$ property: clearly if $p$ is a property
of $h$ then $h'\preceq h$ and all descendants $h''$ of $h$ satisfy
$h'\preceq h \preceq h''$, which means that $p$ is also a property of
$h''$.
Therefore, no successor of $h$ will contradict $p$ 
and, as a consequence, 
properties of the form $\propdesc(h')$ can be ignored 
to determine the successors.  
Formally, 
$\descendants(h) \setminus \hypos(C) = 
\descendants(h) \setminus \hypos(C')$ 
where $C'$ is the subset of properties of $C$ 
that are of type $\neg\propdesc(h')$; 
notice that this does not imply that $C'$ is a conflict.  

Now, let $h$ and $h'$ be two hypotheses.  
We write $h \otimes h'$ for
the set of least common descendants of $h$ and $h'$, 
i.e., $h \otimes h' = \min_{\preceq} 
\left( \descendants(h) \cap \descendants(h') \right)$.  
The following result holds: 
\begin{lemm}\label{lemm::conflictsuccessors}
  Let $S$ be a set of hypotheses 
  and let $C_S = \{\neg\propdesc(h') \in \propspace \mid h' \in S\}$ 
  be a set of properties.  
  Let $h$ be a hypothesis.  
  Then, 
  \begin{displaymath}
    \min_{\preceq} (\descendants(h) \setminus \hypos(C_S)) = 
    \min_{\preceq} (\bigcup_{h' \in S} h \otimes h').  
  \end{displaymath}
\end{lemm}

\begin{proof}
  This proof is in two parts: 
  {first, we prove that if $S_1$ covers $S_2$ 
  (i.e., for all hypotheses of $S_2$, 
  there exists a preferred hypothesis in $S_1$) and conversely,}
  then their minimal sets are equals; 
  second, we prove that the two-way coverage holds
  for $S_1 = \descendants(H) \setminus \hypos(C_S)$ 
  and for $S_2 = \bigcup_{h' \in S} h \otimes h'$.  

  Let $S_1$ and $S_2$ be two sets of hypotheses 
  such that 
  $\forall \{i,j\} = \{1,2\}~
  \forall h_i\in S_i ~\exists h_j \in S_j ~h_i \preceq h_j$.  
  Consider an element $h_i \in \min_\preceq(S_i)$; 
  since $h_i \in S_i$, there exists $h_j \in S_j$ 
  such that $h_j \preceq h_i$.  
  Furthermore since $h_j \in S_j$, there exists $h'_i \in S_i$ 
  such that $h'_i \preceq h_j$.  
  Hence $h'_i \preceq h_i$ and therefore $h'_i = h_i$ 
  (if $h'_i \prec h_i$, then $h_i$ would not minimal).  
  Consequently $h'_i \preceq h_j \preceq h_i$ and $h'_i = h_i$, 
  which implies that $h_i = h_j$.  
  Thus $\min_{\preceq} S_1 = \min_{\preceq} S_2$.  

  Assume now $S_1 = \descendants(h) \setminus \hypos(C_S)$ 
  and $S_2 = \bigcup_{h' \in S} h \otimes h'$.  
  We prove that $S_1$ covers $S_2$, and in the next paragraph we prove
  that the converse holds as well.
  Let $h_2 \in S_2$  and let $h' \in S$ be the hypothesis 
  such that $h_2 \in h \otimes h'$, 
  then $h \preceq h_2$ and $h' \preceq h_2$; 
  hence $h_2 \in \descendants(h)$ 
  and $h_2 \not\in \hypos(C_S)$ 
  (since $\hypos(C_S)$ includes $\neg\propdesc(h')$).  

  Let $h_1 \in S_1$ be a hypothesis.  
  By definition $h_1$ is a descendant of $h$ 
  and does not belong to $\hypos(C_S)$; 
  hence there exists $h' \in S$ such that $h_1 \preceq h'$.  
  By definition, 
  $S_2 = h \otimes h'$ contains a hypothesis $h_2$ 
  such that $h_2 \preceq h_1$.  
\end{proof}

Lemma~\ref{lemm::conflictsuccessors} gives us a way 
to compute the set of successors.  
Indeed, it should be clear that $h \otimes h'$ 
is finite for any $h$ and any $h'$ 
since the hypothesis space is a well partial order.  
Therefore, the union in Lemma~\ref{lemm::conflictsuccessors}
can be enumerated and the minimal elements found by
pairwise hypothesis comparisons.

The implementation of operator $\otimes$ is often simple.
We now give concrete realisations for some of 
the hypothesis spaces we introduced.  

In SHS, a hypothesis is a set of faults.  
The single hypothesis $h''$ such that $\{h''\} = h \otimes h'$ 
is then $h'' = h \cup h'$.  

In MHS, a hypothesis associates each fault 
with a number of occurrences.  
Again, $h\otimes h'$ produces a single hypothesis $h''$, 
which is defined by $h''(f) = \max \{h(f),h'(f)\}$ for all fault $f$.  

In SqHS, multiple hypotheses can be 
minimal common descendants of $h$ and $h'$.  
Such hypotheses $h''$ are such that they contain 
all faults in $h$ and $h'$, in the same order.  
The set of hypotheses can be computed by 
progressing
in $h$, $h'$, or in both at the same time (if the current fault is the same), 
until the end of both sequences is reached.  
Certain non-minimal hypotheses may still slip in, 
and must be removed.  
For instance, if $h = [a,b]$ and $h' = [b,c]$, 
the procedure described above would produce: 
$\{[a,b,b,c],[a,b,c,b],[a,b,c],[b,a,b,c],[b,a,c,b],[b,c,a,b]\}$ 
but the result is actually $h \otimes h' = \{[a,b,c],[b,a,c,b],[b,c,a,b]\}$.  


\section{Related Work}
\label{sec::related}
The AI and control communities have developed a wide spectrum of
diagnosis approaches targeting static or dynamic, discrete event,
continuous, or hybrid systems. Obviously, we cannot discuss all of
these. For instance, we do not cover approaches in state
estimation or probabilistic diagnosis whose goal is
to compute a probability distribution on candidates \cite{thorsley-teneketzis::tac::05,stern-etal::aaai::15}.  Instead, we
focus our discussion on the frameworks which ours generalises. This
includes in particular the founding works of
Reiter \cite{reiter::aij::87}, de Kleer and
Williams \cite{dekleer-williams::aij::87}, and approaches that employ
related algorithmic frameworks \cite{feldman-etal::jair::10}.

\subsection{Connection with Reiter's Theory}  

Reiter's work \cite{reiter::aij::87} is a key inspiration to the
present theory.  Similarly to Reiter's, our objective is a general
theory of diagnosis from first principles, which determines the
preferred diagnosis hypotheses solely from the available description
of the system and of its observed behaviour, and which is independent
from the way systems, hypotheses, and observations are represented.

Our work generalises Reiter's in two significant ways. First,
Reiter only considers the set hypothesis space (SHS). This
space has many properties
\cite{staroswiecki-etal::ijacsp::12}, 
which allowed Reiter to propose a more specific implementation of PFS+c
(\texttt{diagnose}).  
SHS is finite, which means that termination is not an issue 
(by no mean does this implies 
that Reiter and other researchers did not try to accelerate termination).  
It is also a lattice, 
i.e., any pair $\{h,h'\}$ of hypotheses has a unique least upper bound 
and a unique greatest lower bound; 
practically, this means that $h \otimes h'$ is always singleton, which
simplifies successor computation.
Finally, and most importantly, 
each hypothesis can be defined as the intersection 
of the set of descendants or non-descendants of hypotheses of depth~1.  
For instance, if $F = \{f_1,f_2,f_3\}$, 
then $\{f_1,f_2\}$ is the unique element 
of $\descendants(\{f_1\}) \cap \descendants(\{f_2\}) \cap 
\left(\hypospace \setminus \descendants(\{f_3\})\right)$.  
Similarly, the set of descendants of any hypothesis
is the intersection of descendants of hypotheses of depth~1: 
$\descendants(\{f_1,f_2\}) = 
\descendants(\{f_1\}) \cap \descendants(\{f_2\})$.  
Practically, this means that there exists a specialised property space
that can be used to uniformly represent all hypotheses and
that leads to conflicts that generalise well across the hypothesis space.
For all these reasons, Reiter did not have to introduce 
the more complex algorithmic machinery we use in this paper.  
However, our theory enables much richer hypotheses spaces to
be considered.

This leads us to the second main difference with Reiter's work:
whilst system-inde\-pen\-den\-ce was one of Reiter's original aims, his
theory was mainly applied to circuits and other static systems
\cite{dague:94}. Dynamic systems and in particular DESs were investigated
using totally different approaches. In part, this can be explained by
the immaturity of available consistency-checking tools for DESs
(model checkers and AI planners) at the time. However, dynamic systems
also naturally lend themselves to diagnostic abstractions richer than the
set hypotheses space, such as considering sequences of fault events
\cite{cordier:thiebaux:94}.

\subsection{Connection with de Kleer's Theory}
\label{sub::related::dekleer}

Reiter's theory applies to weak-fault models, which model only the
correct behavior of components. De Kleer and Williams
\cite{dekleer-williams::aij::87} extended Reiter's work to
strong-fault models, which incorporate information about faulty
behavior. They also used a different computational strategy,
exploiting an
assumption-based truth maintenance system (ATMS)
\cite{dekleer::aij::86}.  Their approach however still assumes the
set hypothesis space.

Strong-fault models bring additional challenges to the development of
a general theory of diagnosis. Weak-fault models have a certain
monotonicity property: if $\delta \preceq h$ and $\delta$ is a
candidate, then $h$ is also a candidate.  
This is one justification for returning the minimal diagnosis: 
it implicitly represents all diagnosis candidates.  
Such a representation however is no
longer possible with strong-fault models, and instead, a new notion
of ``kernel diagnosis'' was introduced \cite{dekleer-etal::aaai::90}.
A kernel diagnosis is the conjunction of descendants and
non-descendants of specified sets of hypotheses,
e.g. $\descendants(\{f_1\}) \cap \descendants(\{f_2\}) \cap
\left(\hypospace \setminus \descendants(\{f_3\})\right) \cap \left(\hypospace \setminus \descendants(\{f_4\})\right)$, and the
diagnosis can be represented by a (finite) set of maximal kernel
diagnoses. Note that although all minimal candidates belong to some
kernel diagnosis, i) this kernel diagnosis is not solely defined by
the minimal candidate and ii) not all kernel diagnoses contain a minimal
candidate.

The generalisation of a kernel diagnosis to a richer
hypothesis space than SHS is not trivial.  For strong-fault models,
the main benefits of representing the diagnosis as a set of kernel
diagnoses over a set of minimal diagnoses are that: i) the candidates
can be easily enumerated; and ii) verifying that a hypothesis is a
candidate is easy.  A kernel diagnosis represented by a set of
properties as defined in the present article satisfies these two
criteria.  However the set of kernel diagnoses may become infinite.
To see this, consider the following example
over a multiset hypothesis space (MHS) 
with two fault events $f_1$ and $f_2$; 
for simplicity a hypothesis will be written $h_{i,j}$ 
which means that fault $f_1$ occurred $i$ times, 
and fault $f_2$ $j$ times.  
We assume that $\Delta = \{ h_{0,j} \mid j \mathrm{\ mod\ } 2 = 1\}$, i.e., $f_1$ did not occur and $f_2$ occurred an odd number of times.
The kernel diagnoses are the following:
\begin{displaymath}
  \descendants(h_{0,1+2i}) 
  \setminus \descendants(h_{1,1+2i})
  \setminus \descendants(h_{0,2+2i}),\quad i \in \mathbf{N}.  
\end{displaymath}
Such a representation of the diagnosis is infinite,
which is why we advocate the computation of the minimal diagnosis.

The second characteristic of the theory 
developed by de Kleer and Williams 
is the use of an ATMS to generate all the maximal conflicts 
before computing the diagnosis.  
ATMSs compute these conflicts by propagating 
the consequences of assumptions on the hypothesis properties.  
However, assuming, as is the case in this article, 
that the conflicts are convex sets of non-candidate hypotheses, 
the set of maximal conflicts may be infinite.  
Consider again the above example,
and let $h_{0,i} \neq h_{0,j}$ be two non-candidate hypotheses.
Clearly both hypotheses cannot be in the same convex conflict
(at least one hypothesis between them is a candidate).
Thus, using an ATMS to pre-generate maximal convex conflicts
is not feasible in the context of more general hypothesis spaces.

Furthermore, even when the conflict set is finite, it can be
prohibitively large and include many conflicts that are not needed to
solve the problem. For instance, many conflicts will discard 
hypotheses that would not be minimal, even if they were candidates.
An example of this, from the example above, is a conflict $C$ 
where $\hypos(C) = \{h_{0,2}\}$.  
Such conflicts are not necessary to compute the minimal diagnosis.
In the PFS algorithm, as well as other algorithms for computing hitting
sets, the incremental generation of ``useful'' conflicts is preferrable.

To avoid computing a potentially exponentially long list 
of minimal candidates, 
Williams and Ragno \cite{williams-ragno::dam::07} 
proposed to compute a subset of candidates 
that optimise some utility function 
(for instance, maximises the probability 
given a priori probabilities on faults).  

\subsection{PLS-like Systems}

Bylander et al. proposed an approach 
that bears some similarities with PLS+r, 
in that it finds any diagnosis candidate 
and then searches for a candidate within its parent set
\cite{bylander-etal::aij::91}.  
It assumes the set hypothesis space, and that the problem has
the monotonicity property ($\delta \preceq h\ \land\ \delta \in \Delta 
\Rightarrow h \in \Delta$), like weak-fault models do.
This algorithm does not return all minimal candidates.  

SAFARI \cite{feldman-etal::jair::10} is a variant of this approach.  
It too assumes the SHS and a weak-fault model.  
The goal of this algorithm is to avoid the memory requirements 
associated with computing all the conflicts, as done with the
ATMS approach, or maintaining an open list of hypotheses, as in
the \texttt{diagnose} algorithm.

SAFARI first computes a diagnostic candidate.  
It then checks whether any parent of the current candidate 
is a candidate, in which case it iteratively 
searches for a candidate parent.  
Because the model is weak-fault, 
this approach is guaranteed 
to return a minimal candidate.  
When a minimal candidate is found, 
a new search is started.  
This approach does not guarantee 
that all minimal candidates will be found.  
Furthermore to speed up the implementation, 
not all the parents are checked: 
the refinement is stopped as soon 
as two parent checks fail.  

\subsection{Explanatory Diagnosis of Discrete-Event Systems}

Recently, Bertoglio et al. \cite{bertoglio-etal::ecai::20,bertoglio-etal::kr::20,lamperti-etal::jair::23}
proposed the \emph{explanatory diagnosis} of discrete event systems.
They compute all possible sequences of faults that are consistent with the observations.
The number of such sequences can be infinite in general,
but they use regular expressions to represent them compactly.
This diagnosis is more informative than the diagnosis traditionally computed for DES
(set of faults).

There are several important differences between their work and ours.
First, they compute the complete diagnosis while we focus on computing the \emph{minimal} diagnosis;
restricting ourselves to minimal diagnosis allows us to use more efficient algorithms,
while Bertoglio et al. must explore all behaviours exhaustively.
Second, they define diagnosis candidates as \emph{sequences} of faults
while we allow for more definitions.
This is not restrictive per se, as the sequences of faults
form the most abstract space,
but this, again, implies that we can use algorithmic improvements specific to our hypothesis space.
Thirdly,  we use an approach based on consistency tests
while Bertoglio et al. compute all behaviours consistent with the observations.
Finally, our approach is not limited to discrete event systems.

\subsection{Navigating the Space  of Plans}

The problem of navigating through the space of possible plans in AI
planning is very similar to the problem of diagnosis of discrete event
systems.  In classical planning, the optimal plan is generally the
least expensive one.  However the preference relation is sometimes
more complex. One example is oversubscription planning, which requires
finding a plan that satisfies all the hard goals and a maximal subset
of soft goals. Because the planner does not know which combination of
soft goals the user would rather see achieved, it should return all
(cost optimal) plans that are non-dominated, i.e., such that no other
plan achieve a superset of goals.

Such problems can be formulated in our framework. The observations are
the language of all plans that reach the hard goals. A ``hypothesis''
associated with a plan represents the subset of soft goals that this
plan achieves.  A hypothesis is preferable to another one if it is a
superset of the latter.  We can then use our search strategies to
efficiently search for solutions.  Eifler et
al. \cite{eifler-etal::aaai::20} propose techniques that are similar
to search over hypothesis space as is done in model-based diagnosis.
However our approach allows for more sophisticated definitions of
classes of plans: rather than two plan belonging to the same class if
they achieve the same soft goals, the user could also be interested in
the order in which these goals are achieved (for instance the order
in which a certain people can be visited). This can be modelled in our
framework as a variant of the Sequence Hypothesis Space in which an
element appears only once.

\subsection{Generation of Conflicts}

The theory of diagnosis from first principles 
relies heavily on the notion of conflicts 
to explore the hypothesis space efficiently.  
Junker presented an algorithm dubbed \textsc{QuickXplain} 
for computing minimal conflicts from a consistency checker
that is ignorant of the underlying problem 
\cite{junker::aaai::04}.  
\textsc{QuickXplain} isolates the subset of properties 
responsible for inconsistency by iteratively splitting 
the set of properties and testing them separately.  

Shchekotykhin et al.{} improved this work 
to produce several conflicts in a single pass
\cite{shchekotykhin-etal::ijcai::15}.  

The applications mentioned in the papers cited above considered the
Set Hypothesis Space but these algorithms are applicable to any
hypothesis space and can be used in our framework to generate
conflicts.

{In the context of heuristic search planning, Steinmetz and
Hoffmann \cite{steinmetz-hoffmann::aij::17} presented a technique to
find conflicts (that are not guaranteed minimal).  A conflict is
a conjunction of facts such that any state that satisfies it is a
dead-end from which the problem goal cannot be reached. Their
algorithm uses the critical path heuristic $h^C$ \cite{haslum:12}
which lower bounds the cost of reaching the goal, as a dead-end
detector, i.e., when $h^C(s)=\infty$, the state $s$ is a dead-end. The
algorithm  incrementally learns the value of the parameter $C$ -- a set of
conjunction of facts --- adding new conjunctions when a dead-end
unrecognised by $h^C$ is found by the search. In our implementation of
a test solver based on heuristic search below, we build on a different
planning heuristic, namely LM-cut.
}

\subsection{Other Diagnosis Approaches}

Besides test-based approaches to diagnosis, 
two different classes of approaches have been developed 
\cite{grastien::dx::13::spectrum}.  

The first, which bears some similarities with the test-based
approach, consists in determining, off-line, a mapping between
assumptions on the diagnosis and patterns satisfied by the
observation.  Implementations include indicators and ARR
\cite{staroswiecki-comtetvarga::automatica::01}, possible conflicts
\cite{pulido-alonsogonzalez::tsmc::04}, chronicles
\cite{cordier-dousson::safeprocess::00}, and, in an extreme
interpretation of this class, the Sampath diagnoser
\cite{sampath-etal::tac::95}.  The problem with approaches of this kind
is the
potentially large (exponential, or worse) number of observation
patterns that need to
be built off-line.

The second approach consists in computing the set of behaviours that
are consistent with the model and the observation, and
extracting the diagnosis information from these behaviours.  The main
issue here is finding a representation of the set of behaviours that
is compact enough and allows fast extraction of the diagnostic
information.  In circuit diagnosis, this approach has been pioneered
by Darwiche and co-authors, and led to a thorough study of
model compilation \cite{darwiche-marquis::jair::02,darwiche::ijcai::11}.  
For DES
diagnosis, this approach has dominated the research landscape
\cite{pencole-cordier::aij::05,su-wonham::tac::05,%
  schumann-etal::aaai::07,kanjohn-grastien::ecai::08,%
  zanella-lamperti::03}. The present paper significantly departs from
existing work on DES diagnosis by offering a generalised test-based theory that
encompasses DES and other types of dynamic systems.


\section{Implementations}
\label{sec::imple}
The framework presented in this paper was initially developed for the
diagnosis of discrete event systems.  In the DES case, the task of the
test solver is to decide if there exists a sequence $\sigma$ of events
that is allowed by the system model ($\sigma \in \MOD$), consistent
with the observation ($\OBS(\sigma)$ holds) and matching a hypothesis
in the test set $H$ ($\hypo(\sigma) = h$ for some $h \in H$).
Realistically, we must assume that the model is given in some compact,
factored representation, such as a network of partially synchronised
automata \cite{pencole-cordier::aij::05}, a Petri
net \cite{benveniste:etal:03} or a modelling formalism using state
variables and actions \cite{haslum:etal:19}. Even if these
representations are in theory equivalent to a single finite automaton,
the exponential size of that automaton means it can never be fully
constructed in practice.  Thus, the test solver must work directly on
the factored representation.  This is the same problem that is faced
in model checking
\cite{clarke-etal::00} and AI planning \cite{ghallab-etal::00}, and
techniques from those areas can be adapted to solve it.

In this section, we present two examples of how test solvers for
DES diagnosis, over different hypothesis spaces, can be implemented.
One implementation uses a reduction to propositional satisfiability
(SAT), while the other uses heuristic state space search.
To ground the discussion, we first introduce a simple, concrete
factored DES representation.

\subsection{Representation of Large DES}
The representation that we will use to describe the test solver
implementations below is a network of partially synchronised automata.
This is a commonly used representation for DES diagnosis 
\cite{zanella-lamperti::03,pencole-cordier::aij::05,su-wonham::tac::05}.

The DES is defined by a set of \emph{components}, $\mathcal{C}$, and
a global alphabet of \emph{events}, $\mathcal{E}$. Each component $c$
is a finite state machine: it has a set of local states $S_c$ and a
local transition relation
$T_c \subseteq S_c \times \mathcal{E}_c \times S_c$,
where $\mathcal{E}_c$ is the set of events that component $c$
participates in.
As usual, $(s, e, s') \in T_c$ means the component can change from
state $s$ to $s'$ on event $e$.  
The global state of the system is the tuple of component states,
and a \emph{global transition} is a set of simultaneous component
transitions. Synchronisation is partial: if $e \not\in \mathcal{E}_c$
then $c$ does not perform a transition when $e$ occurs.
More formally, given a global state $(s_1, \ldots, s_n)$, event
$e$ induces a global transition to a new state $(s_1', \ldots, s_n')$
iff for each component $c_i$, either (i) $(s_i, e, s_i') \in T_{c_i}$, or
(ii) $e \not\in \mathcal{E}_{c_i}$ and $s_i' = s_i$.

A subset $\mathcal{E}_O$ of events are \emph{observable}. In 
the diagnosis problems we consider, the observation is a sequence of
observable events: $\OBS = e_o^1, \ldots, e_o^k$.  Another subset
$\mathcal{F} \subseteq \mathcal{E}$ are designated fault events.

\subsection{Implementation of PFS+ec using SAT}
\label{sub::dessat}

Propositional satisfiability (SAT) is the problem of finding a
satisfying assignment to a propositional logic formula on conjunctive
normal form (CNF), or prove that the formula is inconsistent.
SAT has many appealing characteristics as a basis for implementing
a test solver: modern SAT solvers based on clause learning are very
efficient, both when the answer to the question is positive and
negative, and can easily be modified to return a conflict. Reductions
to SAT have previously been used to solve discrete event reachability
problems for diagnosis \cite{grastien-etal::aaai::07,%
grastien-anbulagan::tac::13}, AI planning
\cite{kautz-selman96} and model checking \cite{biere:etal:03}.

The main disadvantage of reducing the reachability problem to SAT is
that it requires a bound on the ``parallel length'' $n$ of the
sequence $\sigma$ that is sought, and the size of the SAT encoding
grows proportionally to this parameter.\footnote{The SAT encoding
allows parallel execution of non-synchronised local transitions in
separate components. The semantics of such parallel execution is
simple: a parallel set of global transitions is permitted iff every
linearisation of it would be. The purpose of allowing this form of
parallelism is only to reduce the size of the encoding.}
For the benchmark problem that we consider in our experiments
(described in Section \ref{sec:xp:problem}) this is not problematic:
the structure of this benchmark allows us to prove that the maximum
number of local transitions that can take place in any component
between two observable events is at most $7$, and therefore the
parallel length of the sequence is bounded by $7 \times |o|$, where
$|o|$ is the number of observed events.
For diagnosis of DESs where such a bound cannot be proven, however,
this can be an issue.

In order to represent a path of parallel length $n$ (where we
take $n = 7 \times |o|$),
we define SAT variables that model the state of every component 
between every consecutive pair of transitions 
as well as variables that model which event occurred on each transition.  
For each $s \in S_c$ of some component 
and each ``timestep'' $t \in \{0,\dots,n\}$, 
the propositional variable $s@t$ 
will evaluate to \textit{true} iff the state of component $c$ is $s$ 
after the $t$-th transition.  
Similarly, for every event $e \in \mathcal{E}$ 
and every timestep $t \in \{1,\dots,n\}$, 
the propositional variable $e@t$ will evaluate to \textit{true} 
if event $e$ occurred in the $t$-th transition.
For simplicity, we also define the propositional variable $tr@t$ 
which represents whether the (component) transition $tr$ 
was triggered at timestep $t$.  

The SAT clauses are defined to ensure that any solution to the SAT problem 
represents a path that satisfies the following three constraints 
\cite{grastien-etal::aaai::07}: 
(i) it should be allowed by the model; 
(ii) it should be consistent with the observations; 
(iii) its corresponding hypothesis should belong to the specified set $H$.  

The translation of the first constraint into SAT 
is summarised on Table~\ref{tab::modsatconstraints}.  
The first two lines ensure 
that the origin and target states of each transition are satisfied.  
The third line encodes the frame axiom, 
which specifies that a component state changes 
only as an effect of a transition.  
The fourth line is a cardinality constraint \cite{marquessilva-lynce::cp::07} 
which indicates that a component can only be in one state at a time.  
The fifth and sixth lines ensure 
that the transitions and events match 
and the seventh line is a cardinality constraint 
whereby only one event can take place at a time for each component.  
The last line defines the initial state 
(component $c$ starts in state $s_{c0}$).  

\begin{table}[t!]
  \begin{displaymath}
    \begin{array}{|l r|}
    \hline
    \forall c \in \mathcal{C}.\ \forall tr = (s,e,s') \in T_c.\ 
    \forall t \in \{1,\dots,n\}& 
    tr@t \rightarrow s'@t\\
    \forall c \in \mathcal{C}.\ \forall tr = (s,e,s') \in T_c.\ 
    \forall t \in \{1,\dots,n\}& 
    tr@t \rightarrow s@(t-1)\\
    \forall c \in \mathcal{C}.\ \forall s \in S_c.\ 
    \forall t \in \{1,\dots,n\}& 
    (s@t \land \overline{s@(t-1)}) \rightarrow \bigvee_{tr \in T_c} tr@t\\
    \forall c \in \mathcal{C}.\ \forall t \in \{0,\dots,n\}&
    \large{=}_1 \{s@t \mid s \in S_c\}\\
    \forall c \in \mathcal{C}.\ \forall tr = (s,e,s') \in T_c.\ 
    \forall t \in \{1,\dots,n\}& 
    tr@t \rightarrow e@t\\
    \forall c \in \mathcal{C}.\ \forall e \in \mathcal{E}_c.\ 
    \forall t \in \{1,\dots,n\}&
    e@t \rightarrow \bigvee_{tr \in T_c} tr@t\\
    \forall c \in \mathcal{C}.\ \forall t \in \{1,\dots,n\}&
    \large{\le}_1 \{e@t \mid e \in \mathcal{E}_c\}\\
    \forall c \in \mathcal{C}&
    s_{c0}@0\\
    \hline
    \end{array}
  \end{displaymath}

  \caption{Ensuring that the SAT solutions represent paths 
  accepted by the model}
  \label{tab::modsatconstraints}
\end{table}

The second constraint (i.e., that the path matches the observation) 
is very easy to encode.  
Given that the $i$-th observed event took place at timestep $7 \times i$, 
we know which observable events occurred at which timestep.  
This information is simply recorded as unit clauses, 
i.e., if observable event $e$ occurred at timestep $t$ 
the clause $e@t$ is created, 
otherwise the clause $\overline{e@t}$ is created.  
More complex observations can be encoded in SAT, 
for instance if the order between the observed event 
is only partially known \cite{haslum-grastien::spark::11}.  

Finally the last constraint is that the hypothesis 
associated with the path should belong to the specified set.  
Remember that the set is implicitly represented 
by a collection of hypothesis properties.  
We have shown in Section~\ref{sub::prop::dynamic} 
how hypothesis properties can be seen as regular languages 
or intersections of such languages; 
these languages can be represented as finite state machines 
which in turn can be translated to SAT 
similarly to the translation to SAT of the model.  

However, for a given hypothesis space, it is usually possible to find
a simpler, more compact, yet logically equivalent encoding of hypothesis
properties.
Let us first consider the set hypothesis space: The property of
being a descendant from $h \subseteq F$ 
can be represented by the clauses
\begin{displaymath}
  f@1 \lor \dots \lor f@n, \quad \forall f \in h,
\end{displaymath}
which state that the faults $f \in h$ must occur in $\beh$.  
On the other hand, the property of being an ancestor of $h$ 
can be represented by the unit clauses
\begin{displaymath}
  \overline{f@t}, \quad \forall f \in F \setminus h,\ t \in \{1,\dots,n\}, 
\end{displaymath}
which state that the faults $f \not\in h$ should not occur in $\beh$.
For the multiset hypothesis space,
these properties can be represented in a similar way using cardinality
constraints: 
$\beh$ corresponds to a descendant (resp. ancestor) of $h$ 
iff for all fault~$f$, 
$\beh$ exhibits more (resp. less) than $h(f)$ occurrences of $f$.  

The encoding for the sequence hypothesis space 
is more complex.  
Let $h = [f_1,\dots,f_k]$ be a hypothesis 
for which the property $\propdesc(h)$ must be encoded.  
We write $\{h_0,\dots,h_k\}$ the set of prefixes of $h$ 
such that $h_k = h$.  
Consider another hypothesis $h' \succeq h_i$ 
for some $i \in \{0,\dots,k-1\}$, 
and assume $f \in F$ is appended to $h'$; 
then $h'f \succeq h_i$.  
Furthermore if $f = f_{i+1}$ 
then $h'f \succeq h_{i+1}$.  
To model this, we introduce fresh SAT variables $dh_i@t$ 
that evaluate to \textit{true} 
iff the trajectory $\beh$ until timestep $t$ 
corresponds to a hypothesis that is a descendant of $h_i$.  
Clearly, $dh_0@t$ is \textit{true} for all $t$; 
furthermore $dh_i@0$ is \textit{false} for all $i > 0$.  
The value of $dh_i@t$ ($i>0$) can be enforced 
by the following constraints: 
\begin{displaymath}
  dh_i@t\ \longleftrightarrow \ 
  dh_i@(t-1) \lor 
  \left(dh_{i-1}@(t-1) \land f_i@t\right).  
\end{displaymath}

Encoding the ancestor property is more difficult.  
Consider a hypothesis $h' \preceq h_j$ 
for some $j \in \{0,\dots,k\}$, 
and assume $f \in F$ is appended to $h'$; 
then $h'f \preceq h_i$ for any $i$ 
such that $f$ appears in $\{f_{j+1},\dots,f_i\}$.  
The negation of this expression is modelled as follows:
$h'f$ is not an ancestor of $h_i$ 
if $h'$ is not an ancestor of $h_i$ 
or there exists a $0 \le j < i$ 
such that $f \notin \{f_{j+1},\dots,f_i\}$ 
and $h'$ is not an ancestor of $h_j$.  
As was the case for descendant properties, 
we create SAT variables $ah_i@t$.  
For all $i$, $ah_i@0$ is \textit{true}.  
The value of $ah_i@t$ is then ensured by 
\begin{displaymath}
  \overline{ah_i@t}\ \longleftrightarrow\ 
  \overline{ah_i@(t-1)} \lor 
  \bigvee_{j < i}\left(\bigvee_{f \in F \setminus \{f_{j+1},\dots,f_i\}}
  \left(\overline{ah_{j}@(t-1)} \land f@t\right)\right).  
\end{displaymath}

\subsection{Implementation using Heuristic State Space Search}

State space exploration algorithms are widely used in model checking
and AI planning. They construct part of the explicit representation
on-the-fly, while searching for a state satisfying a given goal
condition. The use of heuristic guidance enables these algorithms
to focus the search towards the goal and explore only a very small fraction
of the state space before the goal is found.
Problem-independent heuristics are derived automatically from
the factored problem representation \cite{bonet-geffner01}.

To take advantage of the very effective heuristics and search
algorithms that exist, we need to express the hypothesis test
as a state reachability problem, i.e., as a condition on the goal
state to be found. This is straightforward to do with the help
of some auxiliary components.

The main goal is to find a sequence of events that generates the
observation: Suppose, for simplicity, that the observation is a
sequence of events, $e_o^1, \ldots, e_o^n$.
We add new component $c_o$ with states $0,\ldots,n$, which tracks
how far along the sequence we are. Its local transitions are
$(i-1, e_o^i, i)$; thus, any transition that emits an observable
event will synchronise with $c_o$, ensuring that these events
match the observation. The goal condition is then to reach
$c_o = n$. Transitions that emit an observable event not in the
sequence will never be applicable, and can simply be removed.

The formulation of the hypothesis, or set of hypotheses, to be
tested is more complex. Unlike in the SAT encoding, we cannot
just provide an encoding of each diagnosis property in isolation
and specify the test by their conjunction. Instead, we provide
encodings of two of the diagnostic questions described in
Section \ref{sub::relevantsetsofproperties} that are required
by the \PLS{} and \PFS{} algorithms.
We will use the multiset hypothesis space as the illustrative
example. Encodings of the set hypothesis space are also easy to
define (they work exactly the same but consider only the
presence/absence of each fault, rather than the number of times
it occurred). Encoding tests in the sequence hypothesis space is
much more complicated.

\noindent%
\textit{Question 1:} $\candidate{h}$.\hspace{\labelsep}
Recall that a multiset hypothesis is a mapping
$h : \mathcal{F} \rightarrow \mathbb{N}$.
$h$ is a candidate if there is an event sequence $\sigma$ that
includes each fault $f \in \mathcal{F}$ exactly $h(f)$ times.
We can track the occurrences of fault events in the same way
as the observation: For each fault $f$, introduce a component
$c_f$ with states $0,\ldots,h(f)$, and local transitions
$(i-1,f,i)$. This construction ensures that the sequence contains
no more than $h(f)$ occurences of each fault $f$. Adding $c_f = h(f)$,
for each $f$, to the goal condition also ensures that the sequence
exhibits exactly the specified fault counts.

\noindent%
\textit{Generating conflicts.}\hspace{\labelsep}
A complete search on the formulation above will find an event
sequence witnessing that $h$ is a candidate if such a sequence
exists. If it does not, however, the search will only return the
answer ``no'', after exhausting the reachable fraction of the
state space. To generate a conflict, we need a small modification
to both the encoding and the search algorithm.

We extend each fault-counting component $c_f$ with an extra
state $h(f) + 1$, and the local transitions $(h(f), f, h(f) + 1)$
and $(h(f) + 1, f, h(f) + 1)$. This allows event sequences that
contain more occurrences of faults than $h$ specifies. (We also
remove $c_f = h(f)$ from the goal.)
But we also assign a cost to each transition: the cost is one for
these transitions that correspond to additional faults, and zero
for all other transitions. This means that instead of every
sequence that reaches the goal being a witness for $h$, every
sequence with a total cost of zero is such a witness.

We then run an optimal A$^\star$ search \cite{hart-nilsson-raphael68}
using the admissible LM-Cut heuristic \cite{helmert-domshlak:icaps09:lmcut},
but interrupt the search
as soon as the optimal solution cost is proven to be greater than
zero. At this point, every state on the search frontier (open list)
is either reached by a non-zero cost transition (corresponding to
an additional fault not accounted for by $h$), or has a heuristic
estimate greater than zero, indicating that some additional fault
transition must take place between the state and the goal.
Here, the specific heuristic that we use becomes important:
The LM-Cut heuristic solves a relaxed version of the problem and
finds a collection of sets of transitions (with non-zero cost) such
that at least one transition from every set in the collection must
occur between the current state and the goal. Each such set is what
is known as a \emph{disjunctive action landmark} in the planning
literature.
Thus, this heuristic
tells us not only that some additional fault transition must take
place, but gives us a (typically small) set of which additional
faults. Taking the union of these sets (or the singleton set of the
fault transition already taken) over all states on the search
frontier gives us a set $F'$ of faults such any candidate descendant
of $h$ must include at least one fault in $F'$ in addition to those
accounted for by $h$, and that is our conflict.

\noindent%
\textit{Question 3:} $\coverage{S}$.\hspace{\labelsep}
This question asks whether there is any candidate $h'$ such that
$h \not\preceq h'$ for every $h \in S$. For the multiset hypothesis
space, this means finding an event sequence $\sigma$ such that for
each $h \in S$ there is some fault $f$ that occurs in $\sigma$
strictly fewer times than $h(f)$.
As above, we introduce components $c_f$ to count the number of
fault occurrences in the sequence. We set the maximum count to
$n_f = \max_{h \in S} h(f)$, but add the local transition
$(n_f, f, n_f)$, so that $c_f = n_f$ means ``$n_f$ or more
occurrences of $f$''.
That $h \not\preceq \hypo(\sigma)$ can then be expressed by the
disjunction $\bigvee_{f \in \mathcal{F}} c_f < n_f$.
The goal condition, that $h \not\preceq \hypo(\sigma)$ for all
$h \in S$, is simply the conjunction of these conditions for
all $h \in S$.


\section{Experiments}
\label{sec::xp}
In this section, we apply implementations of different diagnosis
algorithms derived from our theoretical framework to two realistic
diagnosis problems, and benchmark them against other algorithms from
the literature.

\subsection{Competing Algorithms}

We compare the \textsc{sat}-based and planning-based implementations 
of the algorithms presented in this paper 
with existing algorithms from the literature.  

\subsubsection{Diagnoser} 

The seminal work on diagnosis of discrete event systems 
introduced the diagnoser \cite{sampath-etal::tac::95}.  
The diagnoser is a deterministic finite automaton (DFA) 
whose transitions are labeled with observations 
and states are labeled with the diagnosis.  
Given a sequence of observations 
one simply needs to follow the single path labeled by this sequence 
and the diagnosis is the label of the state reached in this way.  
There are several issues with the diagnoser 
that prevented the use of this approach.  

First its size (the number of states of the DFA) 
is exponential in the number of states of the model 
and double exponential in the number of faults 
(for the set hypothesis space) \cite{rintanen::ijcai::07}.  
{For the power network that we use as a benchmark in~\ref{sub::power}, the average
number of possible fault events per component is 9, and the average
number of states per component is well over 100; the number of
components is over 10,000. A Sampath diagnoser for this system will
have over $100^{10,000} \times 2^{(9 * 10,000)} \simeq 10^{50,000}$
states.} This method is therefore inapplicable but for small systems 
or systems that can be strongly abstracted.

Second the diagnoser is originally designed 
for totally ordered observations.  
In our application many observations have the same time stamp, 
meaning that the order in which they were emitted is unknown.  
The diagnoser, as presented by Sampath et al., 
can certainly be adapted to account for bounded partial observable order 
but this would augment the size of the diagnoser 
to additional orders of magnitude.  

Third the approach is not applicable to infinite hypothesis spaces 
since the DFA would be infinitely large.  

\subsubsection{Automata}

Automata-based approaches consist in computing an implicit representation 
of the set of all sequences (traces) of events consistent 
with both the model and the observations.  
This representation is useful if it allows one
to quickly infer some information about the actual system behaviour.  
For instance the implicit representation as a single finite state machine 
whose language is exactly said set 
makes it easy (in linear time) to decide whether a specific fault 
could have or definitely has occurred.  

A significant part of the work in discrete event systems 
aims at finding such representations that are compact, 
that can be computed quickly, and that allow for fast inferences~%
\cite{su-wonham::tac::05,pencole-cordier::aij::05,%
cordier-grastien::ijcai::07}.  

We chose an approach based on junction trees 
\cite{kanjohn-grastien::ecai::08}, the state-of-the-art 
in automata-based diagnosis of discrete event systems.  
This approach is based on the property 
that local consistency in tree structures 
is equivalent to global consistency, 
which result we explain now.  

Consider a finite set $S$ of automata 
that implicitely represents the automaton $A$ 
obtained by standard synchronisation of the automata in $S$.  
Each automaton $A_i \in S$ of this set is characterised 
by a set of events $E_i$.  
A property of this setting is that every trace 
obtained by projecting a trace of $A$ on $E_i$ 
is a trace of $A_i$; 
intuitively this means that a sequence of events allowed by $A$ 
is (by definition of the synchronisation) allowed by every $A_i$.  
The converse, the property of global consistency, is generally not true: 
$A_i$ could contain traces 
that are the synchronisation of no trace from $A$.  
Global consistency is a very powerful property, 
because it allows us to answer many questions regarding $A$ 
by only using $S$ (typically, 
questions such as whether a given fault certainly/possibly occurred).  
In general global consistency can only be obtained by computing $A$ 
and then projecting $A$ on every set of events $E_i$ 
(in case this type of operations is repeated several times, 
a minimisation operation is necessary to reduce space explosion issues); 
this is computationally infeasible outside trivial problems.  

Local consistency is the property 
that every pair of automata in $S$ is consistent, 
i.e., the property of consistency holds for the set $\set{A_i,A_j}$ 
and the synchronisation of $A_i$ and $A_j$.  
Local consistency does not imply global consistency.  
It is now possible to view the set $S$ as a graph, 
where each node maps to an automaton 
and there is a path between two automata that share an event 
($E_i \cap E_j = E_{ij} \neq \emptyset$) 
such that all automata on this path also share these events $E_{ij}$.  
If this graph is a tree, 
then local consistency of $S$ implies global consistency.  
In other words global consistency of $S$ can be achieved 
without computing the automaton $A$.  

There remains the issue of making $S$ represent a tree.  
A technique used to transform an arbitrary graph into a tree 
is to construct a hyper-graph 
where the hyper-nodes are subsets of nodes of the original graph: 
a junction tree \cite{jensen-jensen::uai::94}, aka decomposition tree.  
Accordingly, a new set $S'$ of automata is defined 
whose automata $A'_i$ are defined 
as the synchonisation of subsets of $S$.  
In order to reduce the cost of these synchronisations 
the junction tree should have hyper-nodes of minimal cardinality.  
The decision problem associated with finding the optimal junction tree 
is NP-hard but there exist polynomial algorithms 
that provide good trees \cite{kjaerulff::report::90}.  

We start of with $S$ defined as the set of ``local diagnoses'' 
where each local diagnosis is the synchronisation of each component's model
with its local observations.  
The local observations are not totally independent 
but are defined as batches of independent events.  
Therefore each batch is separated by a synchronisation tick 
that ensures that two ordered observations are indeed ordered.  

{
From a tree-shaped locally consistent representation $S'$, 
one needs to extract the minimal diagnosis.
Assuming that the hypothesis space is defined over a subset $F$ of fault events 
(as is the case with SHS, MHS, and SqHS), 
one option would be to compute the language $\mathcal{L}_F$, 
defined as the projection of the language of $S'$ onto $F$,
and then extract its minimal words, 
a problem similar to the \emph{enumeration problem} \cite{ackerman2009efficient}.
How to perform it efficiently given our definition of minimality, 
and how to perform it without explicitly computing $\mathcal{L}_F$
is an open question.
For this reason, we only provide the runtime for computing $S'$ 
as it gives us a good estimate of the overall performance of this approach.}

\subsubsection{BDD}

A different approach to diagnosis of discrete event system 
consists in i) embedding in the system state the diagnostic hypothesis 
associated with the paths that lead to this state  
and ii) computing the set of states (``belief state'') 
that the system may be in after generating the observations.  
The diagnosis is the set of hypotheses 
that label some state of the final belief state.  

The first point is rather easy to solve for some hypothesis spaces.  
For the set hypothesis space simply add 
a state variable $v_f$ for each fault $f$ 
that records the past occurrence of $f$: 
$v_f$ is false in the initial state 
and it switches to true whenever a transition labeled by $f$ 
is encountered on the path.  
Other hypothesis spaces could be defined as easily, 
but the problem is that 
it requires an infinite number of state variables in general.  
It seems that there is no practical upper bound on this number 
but for trivial problems.  

The second point can be described easily.  
The model can be rewritten as a function 
that associates every observable event $o$ 
with a set $T_o$ of pairs of states 
(the event $o$ is generated 
only when the state changes from $q$ to $q'$, 
where $\langle q,q'\rangle \in T$) 
as well as a set $T_\epsilon$ of pairs for unobservable transitions.  
Starting from a given set of states $\mathcal{B}$, 
the set of states reached by any number of unobservable events, 
written $\mathit{silent}(\mathcal{B})$, 
is the minimal set of states that satisfies 
$\mathcal{B} \subseteq \mathit{silent}(\mathcal{B})$ 
and $q \in \mathit{silent}(\mathcal{B}) \ 
\land\ \langle q,q'\rangle \in T_\epsilon
\Rightarrow q' \in \mathit{silent}(\mathcal{B})$; 
this set can be easily obtained by adding to $\mathcal{B}$ 
states $q'$ as defined above until the set remains stable.  
Starting from a set of states $\mathcal{B}$, 
the set of states reached by a single observable event $o$, 
written $\mathit{next}_o(\mathcal{B})$, 
is the set of states defined by the relation $T_o$: 
$\set{q' \mid \exists q \in \mathcal{B}.\ \langle q,q'\rangle \in T_o}$.  

We first assume that the observations 
are just a sequence $o_1,\dots,o_k$ of observed events.  
The belief state at the end of the sequence of observations 
can be computed incrementally 
by alternating the two functions presented before: 
$\mathcal{B} = \mathit{silent} \circ \mathit{next}_{o_1} \circ 
\mathit{silent} \cdots \mathit{silent} 
\circ \mathit{next}_{o_k} \circ \mathit{silent} (\mathcal{B}_0)$ 
where $\mathcal{B}_0$ is the initial belief state.  

Our observations are not a single sequence of observed events: 
the order between some observation fragments are unknown.  
One way to solve this issue 
is by computing all possible sequences of observations 
and computing the union of the belief states obtained for each sequence.  
We use a more sophisticated approach.  
Because the observations are batches of unordered events, 
we compute, for each batch $b_j$, all possible sequences; 
we compute then the belief state from $\mathcal{B}_{j-1}$ 
for each sequence, and we obtain the belief state at the end 
of batch $b_j$ as the union of the belief state for each sequence.  

So far we have not described how the belief states are represented.  
Because a state is an assignment of state variables to Boolean values, 
a state can be seen as a formula in propositional logic.  
A set of states is also a formula 
and the union (resp. the intersection) of two sets 
are implemented as the logical disjunction (resp. the conjunction).  
Sets of pairs of states as $T$ can also be represented as a formula, 
but this requires a copy $v'$ for each state variable $v$.  
The set of states $q'$ associated 
with at least one state of a specified set $Q$, 
formally $\set{q' \mid \exists q \in Q.\ \langle q,q'\rangle \in T}$, 
can be represented by the propositional formula: 
$\exists V.\ (\Phi_T \land \Phi_Q)[V'/V]$ 
where $V'$ is the list of copied variables, 
$\Phi_T$ is the formula representing $T$, 
$\Phi_Q$ is the formula representing $Q$, 
and $[V'/V]$ is the operation that consists 
in renaming in a formula all variables $v'$ with $v$.  

Practically, for applications as model checking 
\cite{burch-etal::tcad::94}, 
classical planning \cite{kissman-edelkamp::aaai::11}, 
and diagnosis \cite{schumann-etal::aaai::07}, 
these formulas are represented using BDDs \cite{bryant::tc::86}.  

Finally one important issue when using BDDs is that of variable order.  
We make sure that every state variable $v$ is followed by its copy $v'$.  
Furthermore we define all variables of each component 
in a single sequence.  

\subsection{Setup}\label{sec:xp1:setup}

We benchmark six different implementations of algorithms derived from
our framework. These are: \PFS+ec using the SAT-based test solver,
applied to the set, multiset and sequence hypothesis spaces; \PFS+c
using the test solver based on heuristic search, applied to the set
hypothesis space; and \PLS{} using the heuristic search-based test
solver, applied to the set and multiset hypothesis spaces. Recall that
the basic version of \PFS{}, without the essentiality test, is only
guaranteed to terminate when used with the finite set hypothesis
space.
Code can be downloaded here: \url{github.com/alban-grastien/diagfwork}.

In addition, we compare the performance of these algorithms 
with two diagnosis methods presented in the previous subsection: 
the junction tree (\JT) approach and the BDD-based (\BDD) approach.  

\JT{}, \BDD{}, and the \PFS{} variants using the SAT-based test solver
are implemented in Java. The SAT-based test solver itself is a version
of \texttt{minisat 2.0} \cite{een:soerensson:03}, modified to return conflicts, and is implemented
in C. The \PFS{} and \PLS{} variants using the heuristic search-based
solver are implemented in Lisp; the test solver is based on the
\texttt{HSP*} AI planner \cite{haslum:08}, which is implemented in C++.
A new test solver instance is invoked for each test, without reuse of
information from previous tests. For the SAT-based solver, it is likely
that using incremental SAT \cite{hooker::jlp::93} could improve the
aggregate performance over multiple tests.
Remember, as we discussed in Subsection~\ref{sub::expressiveness}, 
that computing the diagnosis in the set, multiset and sequence
hypothesis spaces is increasingly harder.

\subsection{First Benchmark: Diagnosis of a Power Transmission Network}
\label{sub::power}

\subsubsection{The Diagnosis Problem}\label{sec:xp:problem}

The problem we consider is that of intelligent alarm processing
for a power transmission network, as introduced by Bauer et al.
\cite{bauer-etal::dx::11}. The observations are alarms, generated
by equipment in the network such as protection devices, switchgear,
voltage and current monitors, etc. The objective of intelligent
alarm processing is to reduce the volume of alarms, which can get
very high, particularly in severe fault situations, by determining
which alarms are ``secondary'', meaning they can be explained as
follow-on effects of others. This is not simply a function of the
alarm itself, but 
depends on the context. As a simple example, if we
can deduce that a power line has become isolated, then an alarm
indicating low or zero voltage on that line is secondary (implied
by the fact that the line is isolated); but in other circumstances,
a low voltage alarm can be the primary indicator of a fault.

The power network is modelled, abstractly, as a discrete event system.
The number of states in each component ranges between $8$ and $1,024$,
with most components having well over a hundred states.
The entire network has over $10,000$ components, but for each problem
instance (partially ordered set of alarms), only a subset of components
are relevant to reasoning about that set of alarms; the number varies
between $2$ and $104$ components in the benchmark problem set.
The initial state is only partially known, and certain components have
up to $128$ initial states. There are 129 instances in the benchmark set,
and the number of observations (alarms) in each ranges from $2$ to $146$.

\subsubsection{Results}\label{sec:xp1:results}

A summary of results, in the form of runtime distributions, is shown
in Figure~\ref{fig::plot}.

\begin{figure}[ht]
  \includegraphics[width=\linewidth]{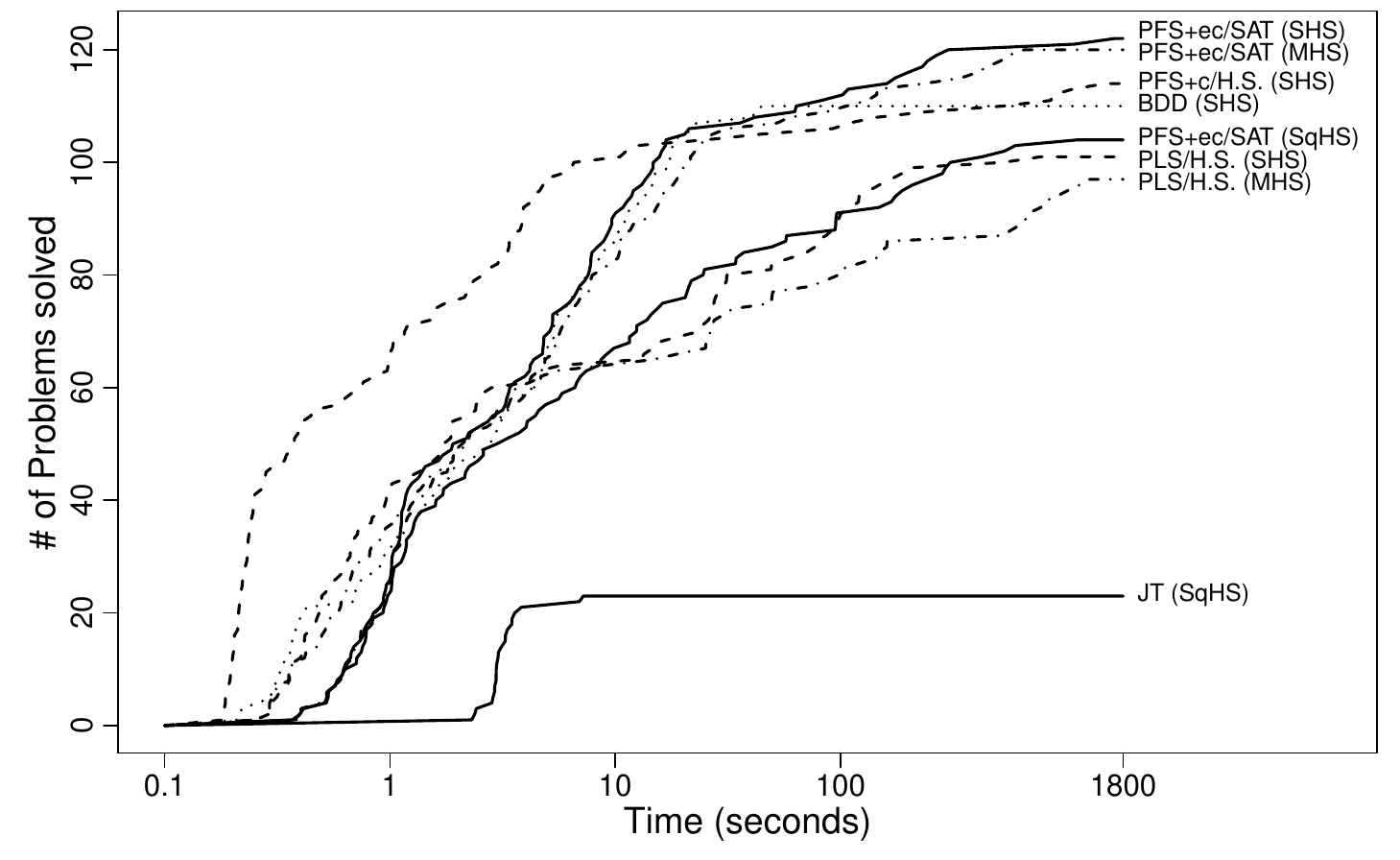}
  \caption{Runtime distribution (number of problems solved vs.\ time limit)
    for all diagnosis algorithms compared in the experiment.}
  \label{fig::plot}
\end{figure}

The complexity of the benchmark instances varies significantly.
Many problems are quite simple, but the complexity rises sharply
in the larger instances. Thus, solving a handful more problems 
is actually a pretty good result.

\JT{} solves only $23$ out of the $129$ instances.  
As soon as the problem includes a transmission line, 
the problem becomes too hard: 
the transmission line component has $1,024$ states, and
$64$ possible initial states, 
which makes the automata determinisation 
required by \JT{} too expensive.

Comparing all the diagnosers operating on the set hypothesis space (SHS),
\PFS{}, with both test solvers, solves more problems than \BDD{} (4 more
with the heuristic search-based solver, 12 more with the SAT-based solver),
which in turn solves 9 more problems than \PLS.
However, it is worth noting that \PFS{} and \PLS{} can return some
diagnosis candidates even when they fail to complete within the given
time limit. All candidates found by \PFS{} are minimal, and so form
a subset of the minimal diagnosis. The instances not solved by
\PFS+ec/SAT (SHS) are also not solved by any other diagnoser, so we cannot
determine how much of the minimal diagnosis has been found.
Concerning \PLS/H.S., in $17\%$ of the instances that it does not solve but
for which the minimal diagnosis is known (because they are solved by
some other diagnoser), the candidate set found by \PLS{}/H.S. is in fact
the minimal diagnosis; it is only the last test, proving that there is no
uncovered candidate, that fails to finish. This can be attributed to
the asymmetric performance of the heuristic search-based test solver:
heuristically guided state space search can be quite effective at finding
a solution when one exists, but is generally no more efficient than blind
search at proving that no solution exists.

It is also interesting to note that the performance of \PFS{}+ec in
the three different hypothesis spaces (SHS, MHS and SqHS) follows the
expected hierarchy of problem hardness: fewer instances are solved in
the sequence hypothesis space, which is a harder diagnosis problem,
than in the easier multiset hypothesis space, and still more instances
are solved in the easiest, the set hypothesis space, though the
difference between MHS and SHS is only two problems.
It turns out that most problem instances 
have the same number of minimal candidates 
for these hypothesis spaces.  
Only two instances solved by both diagnosers show different numbers: 
problem \texttt{chunk-105}, for example,
has two minimal MHS candidates,
$\{\texttt{Line\_X9\_X10.fault}\rightarrow1, 
\texttt{Breaker\_X1\_X2.fault}\rightarrow1\}$ 
and $\{\texttt{Breaker\_X1\_X2.fault}\rightarrow2\}$,
which lead to a single minimal SHS candidate,
$\{\texttt{Breaker\_X1\_X2.fault}\}$.  
Because the size of the minimal diagnoses are similar, 
the number of tests is also very similar 
and incurs only a small penalty for \PFS{} (MHS).
On the contrary, because MHS tests are more precise 
(specifying the exact number of faults), 
and because \PFS{} does not use incremental solving, 
each individual MHS test may be easier to solve.  

\subsection{Second Benchmark: The Labour Market Database}

\subsubsection{The Diagnosis Problem}

This diagnosis problem is based on the data cleansing problem
that we already discussed in \ref{sub::motiv::conformancechecking}.

Specifically, we consider the database provided by \cite{boselli-etal::icaps::14}
that records the employment history in the Italian Labour Market.
This history is subject to logical and legal constraints
(e.g., a job can end only after it has started; a person cannot hold two full-time jobs at the same time).
Data cleansing is the problem of correcting the database to restore its integrity.

Because the constraints apply to each person individually,
we created one problem centered around each person whose history does not satisfy the constraints.
For each problem, we considered all the relevant environment,
in particular the list of employers mentioned in this person's records.
For employers, employees, and jobs, we built generic automata modelling
all histories consistent with the integrity rules.
We further added faulty transitions modelling how the records could be incorrectly inserted into the database,
e.g., transitions representing the fact that an employment cessation was filed for the wrong employer.
A diagnosis is then a sequence of such incorrect operations.

We end up with 600 problems.
The systems are fairly small:
in the worst case, a worker was in contact with five different employers,
which translates into six automata with no more than six states each,
and up to $46$ events per component.

\subsubsection{Results}

A summary of results, in the form of runtime distributions, is shown
in Figure~\ref{fig::plot2}.

\begin{figure}[p]
  \includegraphics[width=\linewidth]{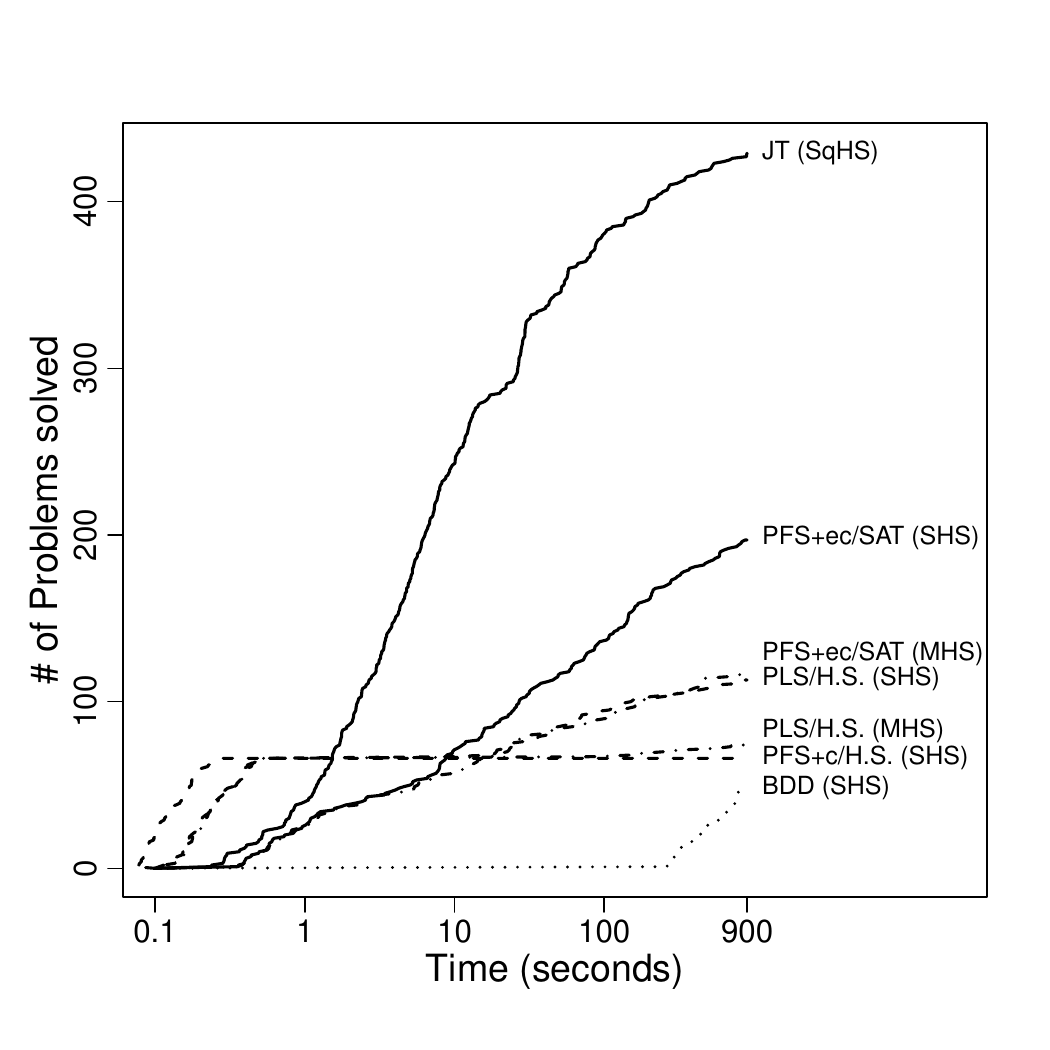}
  \caption{Runtime distribution (number of problems solved vs.{} time limit)
    for all diagnosis algorithms compared in the experiment.}
  \label{fig::plot2}
\end{figure}

The maximum number of minimal candidates in any of the solved instances is $450$,
and the maximum number of faults in such a candidate is $12$.
These numbers are very high, and suggest that the problem definition could be refined.
For instance, in the Set Hypothesis Space,
the preference relation could be enriched by saying
that a hypothesis is preferred over another hypothesis that contains two more faults.
This type of constraint can be easily handled by our framework.

The profile of the algorithms' performance is very different from the first experiments.
We believe that this is due to the features of the problems
that differ significantly from the power network domain.

The Junction Tree algorithm is able to solve a large majority of the instances.
This is due to the fairly small number of states in the diagnosed system.
As a consequence, the necessary operations, such as the automata determinisations,
are relatively quick.
On the other side of the spectrum, the approach based on BDD is able to solve
only a small number of instances;
this is due to the large number of events and transitions,
as well as the number of fault events, that makes each iteration very expensive.
For most instances in which we let the BDD-based diagnoser run longer than 900s,
the computer ran out of memory,
which suggests that this approach will not be able to catch up with the other approaches
beyond the time limit.

Comparing the different algorithms presented in this paper,
we see that PFS is still better, in particular when combined with SAT.
The performance of PFS and PLS are however similar for an oracle using heuristic search planning.


\section{Conclusion and Future Work} 
\label{sec::future}
Prior to our work, diagnosis of discrete event systems has followed its own path,
distinct from that initiated by de Kleer, Reiter, and Williams
for diagnosis for static systems \cite{reiter::aij::87,dekleer-williams::aij::87}. 

In this article, we extended the consistency-based theory of model based diagnosis
to handle diagnosis of systems beyond these static ones.
We showed how to apply the consistency-based approach to all types of systems,
notably discrete event systems and hybrid dynamic ones.
We showed that, for such systems, diagnosis can be computed via a series of consistency tests
that each decide whether the model allows for a behaviour
that i) satisfies certain specified assumptions and ii) agrees with the observations.
We also showed how to perform these tests in practice,
e.g., by using propositional SAT or classical planning.

Extending the consistency-based diagnosis approach to a larger class of systems,
and, in particular, to dynamic systems
prompted us to consider more elaborate definitions of diagnosis and minimal diagnosis.
Some applications, for instance, require us to determine the number or order of fault occurrences.
In other applications, 
certain faults are orders of magnitude less likely than others
and should therefore be ignored whenever more likely behaviours exist,
which leads us to unconventional definitions of preferred hypotheses.
These diagnosis problems are not trivial to solve a priori
as they now feature an infinite search space,
but we showed that our theory can easily handle them
as it only requires us to specify the assumptions appropriately in the consistency tests.
Specifically, as we proved, each assumption should indicate
that the diagnosis hypothesis of the behaviour that the test is looking for
should be better, not better, worse, or not worse than a specified diagnosis hypothesis.
We then just need the test solver to be able to express these assumptions.

We proposed several strategies to generate the diagnosis tests
and showed properties and termination conditions for these strategies.
We also extended the definition of conflict, a central concept in model based diagnosis,
to align with our general theory.

\paragraph{}
Since the beginning of this work, we applied this theory to a range of applications.
We used this theory in combination with SAT modulo theory (SMT)
to diagnose hybrid systems \cite{grastien::ecai::14};
and with model checking to diagnose timed automata \cite{feng-grastien::dx::20}.
The consistency-based approach also allowed us to reason about the observations themselves
and compute a subset of observations that are sufficient
to derive the diagnosis \cite{christopher-etal::cdc::14}.

There has been recently an increased interest in planning for richer problems
involving constraints similar to the ones used in diagnosis tests.
This is motivated by a range of applications:
\begin{description}
\item[Top-$k$ Planning]
  computes several plans that need to be significantly different
  \cite{nguyen2012generating,katz2020reshaping}.
\item[Legibility]
  asks for a plan whose purpose is clear for the observer
  \cite{chakraborti-etal::icaps::19}.
\item[Normative Constraint Signalling]
  requires the plan to communicate to a (partial) observer
  that it follows some normative constraints
  \cite{grastien-etal::aies::21}.
\item[Model Reconciliation]
  assumes two agents with different models of the world,
  and searches for a minimal change to one model
  so that the agents agree on the optimal plan
  \cite{chakraborti-etal::ijcai::17}.
\item[Goal Recognition]
  is the problem of determining what an agent is trying to achieve
  \cite{pereira2019online}.
\item[Plan Explanation]
  provides an explanation alongside the plan that justifies
  why there is no better plan than the proposed one
  \cite{eifler-etal::aaai::20}.
\end{description}
All these problems require reasoning about plans with similar or different properties.
The search strategies developed in this paper can be used
to help some these problems.


\bibliographystyle{theapa}
\bibliography{bib}

\end{document}